\documentclass{article}

\usepackage[final,nonatbib]{neurips_2018}
\usepackage[numbers]{natbib}

\usepackage[utf8]{inputenc} 
\usepackage[T1]{fontenc}    
\usepackage{hyperref}       
\usepackage{url}            
\usepackage{booktabs}       
\usepackage{amsfonts}       
\usepackage{nicefrac}       
\usepackage{microtype}      

\usepackage{amsmath}
\usepackage{amsfonts}
\usepackage{amsthm}
\usepackage{amssymb}
\usepackage{bbm}
\usepackage{commath}

\usepackage{tikz}
\usepackage{tkz-graph}
\usetikzlibrary{patterns}
\usepackage[title,titletoc]{appendix}
\usepackage{color}
\usepackage{algorithm}
\usepackage[noend]{algpseudocode}	
\usepackage{multirow}
\usepackage{mymacros}
\usepackage{array}

\allowdisplaybreaks 

\title{Computationally and statistically efficient learning of causal Bayes nets using path queries}

\author{
	Kevin Bello \\
	Department of Computer Science\\
	Purdue University\\
	West Lafayette, IN, USA \\
	\texttt{kbellome@purdue.edu} \\
\And
	Jean Honorio \\
	Department of Computer Science\\
	Purdue University\\
	West Lafayette, IN, USA \\
	\texttt{jhonorio@purdue.edu} \\
}

\begin{document}

\maketitle

\begin{abstract}
Causal discovery from empirical data is a fundamental problem in many scientific domains. 
Observational data allows for identifiability only up to Markov equivalence class. 
In this paper we first propose a polynomial time algorithm for learning the exact correctly-oriented structure of the transitive reduction of any causal Bayesian network with high probability, by using \emph{interventional path queries}. 
Each path query takes as input an origin node and a target node, and answers whether there is a directed path from the origin to the target. 
This is done by \emph{intervening} on the origin node and observing samples from the target node. 
We theoretically  show the logarithmic sample complexity for the size of interventional data per path query, for continuous and discrete networks. 
We then show how to learn the \emph{transitive} edges using also logarithmic sample complexity (albeit in time exponential in the maximum number of parents for discrete networks), which allows us to learn the full network.
We further extend our work by reducing the number of interventional path queries for learning rooted trees.
We also provide an analysis of imperfect interventions.
\end{abstract}

\section{Introduction}

\paragraph{Motivation.}

Scientists in diverse areas (e.g., epidemiology, economics, etc.) aim to unveil causal relationships within variables from collected data. For instance, biologists try to discover the causal relationships between genes. By providing a specific treatment to a particular gene (origin), one can observe whether there is an effect in another gene (target). This effect can be either direct (if the two genes are connected with a directed edge) or indirect (if there is a directed path from the origin to the target gene).

Bayesian networks (BNs) are powerful representations of joint probability distributions. BNs are also used to describe causal relationships among variables \citep{Koller09}. The structure of a \textit{causal} BN (CBN) is represented by a directed acyclic graph (DAG), where nodes represent random variables, and an edge between two nodes $X$ and $Y$ (i.e., $X \to Y$) represents that the former ($X$) is a direct cause of the latter ($Y$). Learning the DAG structure of a CBN is of much relevance in several domains, and is a problem that has long been studied during the last decades. 

From \textit{observational} data alone (i.e., \textit{passively} observed data from an undisturbed system), DAGs are only identifiable up to Markov equivalence.\footnote{Two graphs are Markov equivalent if they imply the same set of (conditional) independences. In general, two graphs are Markov equivalent iff they have the same structure ignoring arc directions, and have the same v-structures \citep{Verma:1990:ESC:647233.719736}. (A v-structure consists of converging directed edges into the same node, such as $X \to Y \gets Z$.)} 
However, since our goal is causal discovery, this is inadequate as two BNs might be Markov equivalent and yet make different predictions about the consequences of interventions (e.g., $X \gets Y$ and $X \to Y$ are Markov equivalent, but make very different assertions about the effect of changing $X$ on $Y$). 
In general, the only way to distinguish causal graphs from the same Markov equivalence class is to use \textit{interventional} data \citep{hauser2012two, he2008active, murphy2001active}. 
This data is produced after performing an  experiment (intervention) \citep{Pearl:2009:CMR:1642718}, in which one or several random variables are forced to take some specific values, irrespective of their causal parents.

\paragraph{Related work.}
Several methods have been proposed for learning the structure of Bayesian networks from \textit{observational} data.
Approaches ranging from score-maximizing heuristics, exact exponential-time score-maximizing, ordering-based search methods using MCMC, and test-based methods have been developed to name a few. The umbrella of tools for structure learning of Bayesian networks go from exact methods (exponential-time with convergence/consistency guarantees) to heuristics methods (polynomial-time without any convergence/consistency guarantee).
\cite{Hoffgen93} provide a score-maximizing algorithm that is likelihood consistent, but that needs super-exponential time.
\cite{Spirtes00,Cheng02} provide polynomial-time test-based methods that are structure consistent, but results hold only in the infinite-sample limit (i.e., when given an infinite number of samples).
\cite{Chickering02b} show that greedy hill-climbing is structure consistent in the infinite sample limit, with unbounded time.
\cite{Zuk06} show structure consistency of a single network and do not provide uniform consistency for all candidate networks (the authors discuss the issue of not using the union bound in their manuscript).
From the \textit{active learning} literature, most of the works first find a Markov equivalence class (or assume that they have one) from purely \textit{observational} data and then \redcolor{orient} the edges by using as few \textit{interventions} as possible.
\cite{murphy2001active, tong2001active} propose an exponential-time Bayesian approach relying on structural priors and MCMC.
\cite{hauser2012two, he2008active, shanmugam2015learning} present methods to find an optimal set of interventions in polynomial time for a class of chordal DAGs. 
Unfortunately, finding the initial Markov equivalence class remains exponential-time for general DAGs \cite{chickering1996learning,Pearl:2009:CMR:1642718}.
\cite{eaton2007exact} propose an exponential-time dynamic programming algorithm for learning DAG structures exactly.
\cite{triantafillou2015constraint} propose a constraint-based method to combine heterogeneous (observational and interventional) datasets but rely on solving instances of the (NP-hard) \emph{boolean satisfiability problem}. 
\cite{Eberhardt05onthe} analyzed the number of interventions sufficient and in the worst-case necessary to determine the structure of any DAG, although no algorithm or sample complexity analysis was provided.
Literature on learning \emph{structural equation models} from observational data, include the work on continuous \citep{peters2014causal,shimizu2006linear} and discrete \citep{peters2010identifying} additive noise models.
Correctness was shown for the continuous case \citep{peters2014causal} but only in the infinite-sample limit.
\cite{kocaoglu2017experimental} propose a method to learn the exact observable graph by using $\bO{\log n}$ multiple-vertex interventions, where $n$ is the number of variables, through the use of pairwise conditional independence test and assuming access to the post-interventional graph.
However the size of the intervened set is $\bO{\nicefrac{n}{2}}$ which leads to a $\bO{2^{\nicefrac{n}{2}}}$ number of experiments in the worst case.
In contrast to this work, we perform single-vertex interventions as a first step and then multiple-vertex interventions while keeping a small sample complexity. 
While this increments the number of interventions to $n$, we have a better control of the number of experiments.

\begin{remark}
	In this paper we consider one intervention as one selection of variables to intervene.
	However, we consider an experiment as the actual setting of values to the variables.
	For example, if a variable $X$ takes $p$ different values, then one experiment is $X$ taking one specific value.
	To intervene one binary variable, it is common to make 2 experiments, one under treatment, and one under no treatment.
\end{remark}
For a discussion of learning from purely interventional data, as well as availability of purely interventional data, see Appendix \ref{app:discussion}.
\paragraph{Contributions.}
We propose a polynomial time algorithm with provable guarantees for exact learning  of the transitive reduction of any CBN by using interventional path queries.
We emphasize that modeling the problem of structure learning of CBNs as a problem of reconstructing a graph using path queries is also part of our contributions.
We analyze the sample complexity for answering every interventional path query and show that for CBNs of discrete random variables with maximum domain size $r$, the sample complexity is $\mathcal{O}(\log(nr))$; whereas for CBNs of sub-Gaussian random variables, the sample complexity is $\mathcal{O}(\sigmasub\log{n})$ where $\sigmasub$ is an upper bound of the variable variances (marginally as well as after interventions). 
Then, we introduce a new type of query to learn the \emph{transitive edges} (i.e., the edges that are part of the true network but not of the transitive reduction), while the learning is not in polynomial-time for discrete CBNs in the worst case (exponential in the maximum number of parents), we show that the sample complexity is still polynomial.
We also present two extensions: for learning rooted trees the number of path queries is reduced to $\mathcal{O}(n\log n)$, which is an improvement from the $n^2$ for general DAGs.
We also provide an analysis of imperfect interventions.
We summarize our main results in Table \ref{tab:summary} and compare them to one of the closest related work \cite{kocaoglu2017experimental}.

\begin{table}[!ht]
\centering
\caption{Here $n$ is the number of variables, $\sigmasub$ is an upper bound of the variable variances (marginally as well as after interventions), $t$ is the maximum number of parents, $r$ is the maximum number of values a discrete variable can take, and $B$ denotes the time complexity of an independence-test oracle. 
Note that $B \in \bO{2^n}$ in the worst case and not $\bO{2^t}$ because \cite{kocaoglu2017experimental} can select an intervention set of $n/2$ nodes (see for example Appendix F.2). 
In this table, \textit{novel} indicates that no prior work provided results on the respective subject.
Finally, C and D denote continuous and discrete variables respectively.
}
\label{tab:summary}
\begin{scriptsize}
\begin{tabular}{@{}ccccc@{}}
\toprule
Graph                                                                                         & Var.                                        & Algorithms &  Sample complexity       & Time complexity       \\ \midrule
\multicolumn{1}{c|}{\multirow{3}{*}{\begin{tabular}[c]{@{}c@{}}General \\ DAGs\end{tabular}}} & \multicolumn{1}{c|}{\multirow{2}{*}{D}}   &   1, 5, 3, 7 (our work)       &           $\bO{n^2 2^t\log (nr)}$ (Novel, see Thms. 1, 3)              &           $\bO{n^2 2^t\log (nr)}$            \\
\multicolumn{1}{c|}{}                                                                         & \multicolumn{1}{c|}{}                            &    1, 3 in [13]     &       -                 &           $\bO{Btn^2\log^2 n}$ ($B \in \bO{2^n}$)            \\
\multicolumn{1}{c|}{}                                                                         & \multicolumn{1}{c|}{\multirow{1}{*}{C}} &    1, 6, 3, 8 (our work)     &           $\bO{n^2 \sigmasub \log n}$     (Novel, see Thms. 2, 4)         &          $\bO{n^2 \sigmasub \log n}$             \\
\multirow{1}{*}{\begin{tabular}[c]{@{}c@{}}Rooted trees\end{tabular}}                       & \multicolumn{1}{c|}{\multirow{1}{*}{D}}   &    See Section \ref{sec:extensions}        &          $\bO{n \log^2 (nr)}$ (Novel, see Section \ref{sec:extensions})              &       $\bO{n \log^2 (nr)}$                \\
\midrule
\midrule 
Graph                                                                                         & Var.                                        & Algorithms & \# of interventions & \# of experiments \\ \midrule
\multicolumn{1}{c|}{\multirow{3}{*}{\begin{tabular}[c]{@{}c@{}}General \\ DAGs\end{tabular}}} & \multicolumn{1}{c|}{\multirow{2}{*}{D}}   &   1, 5, 3, 7 (our work)       &           $\bO{n^2}$              &           $\bO{n^2 2^t}$            \\
\multicolumn{1}{c|}{}                                                                         & \multicolumn{1}{c|}{}                            &    1, 3 in [13]       &          $\bO{\log n}$              &           $\bO{2^n \log n}$ (see Appendix F.2.)            \\
\multicolumn{1}{c|}{}                                                                         & \multicolumn{1}{c|}{\multirow{1}{*}{C}} &    1, 6, 3, 8 (our work)     &           $\bO{n^2}$              &          $\bO{n^2}$             \\
\multirow{1}{*}{\begin{tabular}[c]{@{}c@{}}Rooted trees\end{tabular}}                       & \multicolumn{1}{c|}{\multirow{1}{*}{D}}   &    See Section \ref{sec:extensions}       &          $\bO{n}$               &       $\bO{nr}$                \\
\bottomrule
\end{tabular}
\end{scriptsize}
\end{table}

\section{Preliminaries}
\label{sec:preliminaries}

In this section, we introduce our formal definitions and notations. Vectors and matrices are denoted by lowercase and uppercase bold faced letters respectively. Random variables are denoted by italicized uppercase letters and their values by lowercase italicized letters. Vector $\ell_p$-norms are denoted by $\lVert \cdot \rVert_p$. For matrices, $\lVert \cdot \rVert_{p,q}$ denotes the entrywise $\ell_{p,q}$ norm, i.e., for $\lVert \mathbf{A} \rVert_{p,q} = \| ( \|(  \mathrm{A}_{1,1}, \dots, \mathrm{A}_{m,1})\|_p , \dots, \|(\mathrm{A}_{1,n},\dots,\mathrm{A}_{m,n})\|_p ) \|_q$. 

Let $ \Grm = (\Vrm, \Erm)$ be \textit{directed acyclic graph} (DAG) with vertex set $\Vrm = \{1,\ldots,n\}$ and edge set $\Erm \subset \Vrm \times \Vrm$, where $(i,j) \in \Erm$ implies the edge $i \to j$. For a node $i \in \Vrm$, we denote $\pi_\Grm(i)$ as the parent set of the node $i$. In addition, a directed path of length $k$ from node $i$ to node $j$ is a sequence of nodes $(i, v_1, v_2, \ldots, v_{k-1}, j)$ such that $\{ (i, v_1), (v_1, v_2), \ldots, (v_{k-2}, v_{k-1}), (v_{k-1}, j)\}$ is a subset of the edge set $\Erm$.

Let $\X = \{X_1,\ldots,X_n\}$ be a set of random variables, with each variable $X_i$ taking values in some domain $Dom[X_i]$. A \textit{Bayesian network} (BN) over $\X$ is a pair $\B = (\Grm,\THT)$ that represents a distribution over the joint space of $\X$. 
Here, $\Grm$ is a DAG, whose nodes correspond to the random variables in $\X$ and whose structure encodes conditional independence properties about the joint distribution, while $\THT$ quantifies the network by specifying the \textit{conditional probability distributions} (CPDs) $P(X_i | \Ui)$. 
We use $\Ui$ to denote the set of random variables which are parents of $X_i$. 
A Bayesian network represents a \textit{joint probability distribution} over the set of variables $\X$, i.e., $P(X_1,\ldots,X_n)=\prod_{i=1}^n P(X_i|\Ui)$. 

Viewed as a probabilistic model, a BN can answer any ``conditioning'' query of the form $P(\boldsymbol{Z}|\boldsymbol{E}=\boldsymbol{e})$ where $\boldsymbol{Z}$ and $\boldsymbol{E}$ are sets of random variables and $\boldsymbol{e}$ is an assignment of values to $\boldsymbol{E}$. 
Nonetheless, a BN can also be viewed as a \textit{causal model} or causal BN (CBN) \citep{Pearl:2009:CMR:1642718}. Under this perspective, the CBN can also be used to answer \textit{interventional} queries, which specify probabilities after we intervene in the model, forcibly setting one or more variables to take on particular values. 
The manipulation theorem \citep{Spirtes00,Pearl:2009:CMR:1642718} states that one can compute the consequences of such interventions (perfect interventions) by ``cutting'' all the arcs coming into the nodes which have been clamped by intervention, and then doing typical probabilistic inference in the ``mutilated'' graph (see Figure \ref{fig:intervention} as an example). 
We follow the standard notation \citep{Pearl:2009:CMR:1642718} for denoting the probability distribution of a variable $\j$ after intervening $\i$, that is, $P(\j|do(\i=x_i))$.
In this case, the joint distribution after intervention is given by $P(X_{1},\dots,X_{i-1},X_{i+1},\dots,X_n | do(X_i = x_i)) = \indicator[X_i=x_i] \prod_{j \neq i} P(X_j | \Uj)$.

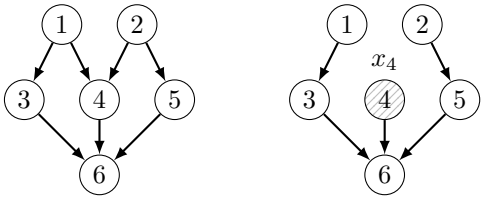
\begin{figure}[ht]
\begin{center}
\begin{tikzpicture}[scale=0.7]
\tikzset{>=latex}
\tikzstyle{vertex}=[circle, fill=white, draw, inner sep=0pt, minimum size=10pt]
\node[vertex][label=center:$1$](x0) at (1.5,3) {};
\node[vertex][label=center:$2$](x1) at (2.5,3) {};
\node[vertex][label=center:$3$](x2) at (1,2) {};
\node[vertex][label=center:$4$](x3) at (2,2) {};
\node[vertex][label=center:$5$](x4) at (3,2) {};
\node[vertex][label=center:$6$](x5) at (2,1) {};
\tikzset{EdgeStyle/.style={->}}
\Edge(x0)(x2)
\Edge(x0)(x3)
\Edge(x1)(x3)
\Edge(x1)(x4)
\Edge(x2)(x5)
\Edge(x3)(x5)
\Edge(x4)(x5)
\end{tikzpicture}
\hspace{0.4in}
\begin{tikzpicture}[scale=0.7]
\tikzset{>=latex}
\tikzstyle{vertex}=[circle, fill=white, draw, inner sep=0pt, minimum size=10pt]
\node[vertex][label=center:$1$](x0) at (1.5,3) {};
\node[vertex][label=center:$2$](x1) at (2.5,3) {};
\node[vertex][label=center:$3$](x2) at (1,2) {};
\node[vertex][label=north:$x_4$, pattern=north east lines, pattern color=gray!50](x3) at (2,2) {$4$};
\node[vertex][label=center:$5$](x4) at (3,2) {};
\node[vertex][label=center:$6$](x5) at (2,1) {};
\tikzset{EdgeStyle/.style={->}}
\Edge(x0)(x2)
\Edge(x1)(x4)
\Edge(x2)(x5)
\Edge(x3)(x5)
\Edge(x4)(x5)
\end{tikzpicture}
\end{center}
\caption{{\small (Left) A CBN of $6$ variables, where the joint distribution, $P(\X)$, is factorized as $\prod_{i} P(X_i | \Ui ).$ (Right) The mutilated CBN after intervening $X_4$ with value $x_4$. Note that the edges $\{(1,4),(2,4)\}$ are not part of the CBN after the intervention, thus, the new joint is $P(\X|do(X_4=x_4)) =  \indicator[X_4=x_4] \prod_{i \neq 4} P(X_i | \Ui ).$ }}
\label{fig:intervention}
\end{figure}

We refer to CBNs in which all random variables $\i$ have finite domain, $Dom[\i]$, as discrete CBNs. In this case, we will denote the probability mass function (PMF) of a random variable as a vector. That is, a PMF, $P(Y)$, can be described as a vector $\pbf(Y) \in [0,1]^{|Dom[Y]|}$ indexed by the elements of $Dom[Y]$, i.e., $\prm_j(Y) = P(Y=j), \forall j \in Dom[Y]$. We refer to networks with variables that have continuous domains as continuous CBNs. 

Next, we formally define transitive edges. 

\begin{definition}[Transitive edge]
\label{def:redundant_edge}
Let $\Grm = (\Vrm, \Erm)$ be a DAG. We say that an edge $(i,j) \in \Erm$ is transitive if there exists a directed path from $i$ to $j$ of length greater than 1.
\end{definition}

The algorithm for removing transitive edges from a DAG is called \emph{transitive reduction} and it was introduced in \cite{aho1972transitive}. The transitive reduction of a DAG $\Grm$, $\TR(\Grm)$, is then $\Grm$ without any of its transitive edges. Our proposed methods also make use of \textit{path queries}, which we define as follows:

\begin{definition}[Path query]
\label{def:query}
Let $\Grm = (\Vrm, \Erm)$ be a DAG. A path query is a function $Q_\Grm: \Vrm  \times \Vrm \rightarrow \{0, 1\}$ such that $Q_\Grm(i, j) = 1$ if there exists a directed path in $\Grm$ from $i$ to $j$, and $Q_\Grm(i, j) = 0$ otherwise.
\end{definition}

\paragraph{General DAGs are identifiable only up to their transitive reduction by using path queries.}

In general, DAGs can be non-identifiable by using path queries. We will use $Q(i,j)$ to denote $Q_\Grm(i,j)$ since for our problem, the DAG $\Grm$ is fixed (but unknown).  For instance, consider the two graphs shown in Figure \ref{fig:dags}. In both cases, we have that $Q(1, 2) = Q(1, 3) = Q(2, 3) = 1$. Thus, by using path queries, it is impossible to discern whether the edge $(1, 3)$ exists or not. 
Later in Subsection \ref{sec:transitive} we focus on the recovery of transitive edges, which requires a different type of query.

\begin{figure}[ht]
\begin{center}
\begin{tikzpicture}[scale=0.7]
\tikzset{>=latex}
\tikzstyle{vertex}=[circle, fill=white, draw, inner sep=0pt, minimum size=10pt]
\node[vertex][label=center:$1$](x0) at (1.5,2) {};
\node[vertex][label=center:$2$](x1) at (1,1) {};
\node[vertex][label=center:$3$](x2) at (2,1) {};
\tikzset{EdgeStyle/.style={->}}
\Edge(x0)(x1)
\Edge(x0)(x2)
\Edge(x1)(x2)
\end{tikzpicture}
\hspace{0.4in}
\begin{tikzpicture}[scale=0.7]
\tikzset{>=latex}
\tikzstyle{vertex}=[circle, fill=white, draw, inner sep=0pt, minimum size=10pt]
\node[vertex][label=center:$1$](x0) at (1.5,2) {};
\node[vertex][label=center:$2$](x1) at (1,1) {};
\node[vertex][label=center:$3$](x2) at (2,1) {};
\tikzset{EdgeStyle/.style={->}}
\Edge(x0)(x1)
\Edge(x1)(x2)
\end{tikzpicture}
\end{center}
\vspace{-0.1in}
\caption{{\small Two directed acyclic graphs that produce the same answers when using path queries.}}
\label{fig:dags}
\end{figure}
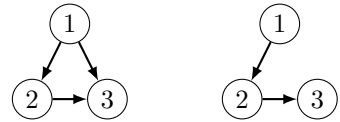

How to answer path queries is a key step in this work. Since we answer path queries by using a finite number of interventional samples, we require a noisy path query, which is defined below.

\begin{definition}[$\delta$-noisy partially-correct path query]
\label{def:noisy_query_1}
Let $\Grm = (\Vrm, \Erm)$ be a DAG, and let $Q_\Grm$ be a path query. Let $\delta \in (0,1)$ be a probability of error. A $\delta$-noisy partially-correct path query is a function $\tilde{Q}_\Grm: \Vrm  \times \Vrm \rightarrow \{0, 1\}$ such that $\tilde{Q}_\Grm(i, j) = Q_\Grm(i,j)$ with probability at least $1-\delta$ if $i \in \pi_\Grm(j)$ or if there is no  directed path from $i$ to $j$.
\end{definition}

We will use the term noisy path query to refer to $\delta$-noisy partially-correct path query. Note that Definition \ref{def:noisy_query_1}  requires a noisy path query to be correct \emph{only in certain cases}, when one variable is parent of the other, or when there is no directed path between them. 
We do not require correctness when there is a directed path between $i$ and $j$ and $i$ is not a parent of $j$, that is, when the path length is greater than 1.
Note that the uncertainty of the exact recovery of the transitive reduction relies on answering multiple noisy path queries.

\subsection{Assumptions}
Here we state the main set of assumptions used throughout our paper.

\begin{assumption}
\label{assump:observ}
Let $\Grm = (\Vrm, \Erm)$ be a DAG. All nodes in $\Grm$ are observable, furthermore, we can perform interventions on any node $i \in \Vrm$.
\end{assumption}

\begin{assumption}[Causal Markov]
\label{assump:markov}
The data is generated from an underlying CBN $(\Grm, \THT)$ over $\X$.
\end{assumption}

\begin{assumption}[Faithfulness]
\label{assump:faith}
The distribution $P$ over $\X$ induced by $(\Grm, \THT)$ satisfies no independences beyond those implied by the structure of $\Grm$.
We also assume faithfulness in the post-interventional distribution.
\end{assumption}

Assumption \ref{assump:observ} implies the availability of purely interventional data, and has been widely used in the literature \citep{murphy2001active, tong2001active, he2008active, hauser2012two, shanmugam2015learning, kocaoglu2017experimental}.
We consider only observed variables because we perform interventions on each node, thus, our method is robust to latent confounders.
(See Appendix \ref{app:latent} for more details).
With Assumption \ref{assump:markov}, we assume that any population produced by a causal graph has the independence relations obtained by applying d-separation to it, while with Assumption \ref{assump:faith}, we ensure that the population has exactly these and no additional independences \citep{Spirtes00, tong2001active, shanmugam2015learning, he2008active, triantafillou2015constraint}.

\section{Algorithms and Sample Complexity}
\label{sec:algorithmsAndSample}

	Next, we present our first set of results and provide a formal analysis on the sample complexity.

	\subsection{Algorithm for Learning the Transitive Reduction of CBNs}

		\cite{kocaoglu2017experimental} show that by using $\BigO{\log n}$ multiple-vertex interventions, one can recover the transitive reduction of a DAG.
		However, in this case, each set of intervened variables has a size of $\ABigO{\nicefrac{n}{2}}$, which means that the method of \cite{kocaoglu2017experimental} has to perform a total of $\ABigO{2^{\nicefrac{n}{2}} \log n}$ experiments, one for each possible setting of the $\ABigO{\nicefrac{n}{2}}$ intervened variables (see an example of this in Appendix \ref{app:examples}). 
		Thus, in this part we work with single-vertex interventions to avoid the exponential number of experiments.
		We can then learn the transitive reduction as follows (see mote details in Appendix \ref{app:algo_transitive_reduction}).
		
		\begin{algorithm2}
		\label{algo:algo1}
		Start with a set of edges $\hat{\Erm} = \varnothing$. 
		Then for each pair of nodes $i,j \in \Vrm$, compute the \textit{noisy path query} $\nQ(i,j)$ and add the edge $(i,j)$ to $\hat{\Erm}$ if the query returns $1$. 
		Finally, compute the transitive reduction of $\hat{\Erm}$ in poly-time \cite{aho1972transitive}, and return $\hat{\Erm}$.
		\end{algorithm2}

		As seen in the next section, each query is computed using single-vertex interventions.
		In fact, for each intervened node, we can compute $n$ queries, i.e., while the number of queries is $n^2$, the number of interventions is $n$.
		This number of single-vertex interventions is necessary in the worst case \cite{Eberhardt05onthe}.
		
		It is natural to ask what would be the benefit of using path queries.
		A query $\nQ(i,j)$ can be interpreted as observing the variable $\j$ after intervening $\i$.
		Under this viewpoint, if one could reduce the number of queries for learning certain classes of graphs, then not only might the number of interventions decrease but the number of variables to observe too.
		That is, if one knows a priori that the topology of the graph belongs to a certain family of graphs then it may be possible to reduce the number of queries (see for example Section \ref{sec:extensions}).
		This is important in practice as both performing interventions and observing variables might be costly.
		We first focus in learning general DAGs, in which a number of $\Omega(n^2)$ path queries is in the worst case necessary for any conceivable algorithm.
		(See Theorems 7 and 8 in \cite{wang2016reconstructing}).
		Later we show that a number of $\BigO{n\log n}$ noisy path queries\footnote{This path query requires a ``stronger'' version of Definition \ref{def:noisy_query_1}. See for instance Definition \ref{def:noisy_path_query}.} suffices for learning rooted trees.

	\subsection{Noisy Path Query Algorithm}

		The next two propositions are important for answering a path query. 
		\begin{proposition} 
		\label{lemma:dseparation}
		Let $\B = (\Grm,\THT)$ be a CBN with $\i,\j \in \X$ being any two random variables in $\Grm$. If there is no directed path from $i$ to $j$ in $\Grm$, then  $P(\j|do(\i=x_i)) = P(\j)$.
		\end{proposition}
		\begin{proposition}
		\label{lemma:probability}
		Let $\B = (\Grm,\THT)$ be a CBN and let $X_i$ and $X_j$ be two random variables in $\Grm$, such that $i \in \pi_\Grm(j)$. Then, there exists $x_i$ and $x_i'$ such that:
		\begin{align*}
		1. \ P(X_j) \neq P(X_j|do(X_i=x_i)) \ \text{ and } \ 2.\ P(\j|do(\i=x_i)) \neq P(\j|do(\i=x'_i))
		\end{align*}
		\end{proposition}
		See Appendix \ref{appendix_proofs} for details of all proofs. 
		Proposition \ref{lemma:probability} motivates the idea that we can search for two different values of $X_i$ to determine the causal dependence on $X_j$ (Claim 2), which is arguably useful for discrete CBNs. Alternatively, we can use the expected value of $\j$, since $\mathbb{E}[\j] \neq \mathbb{E}[\j|do(\i=x_i)]$ implies that $P(\j) \neq P(\j|do(\i=x_i))$ (Claim 1).

		Next, we propose a polynomial time algorithm for answering a noisy path query.  
		Algorithm \ref{algo:query_informal} presents the procedure in an intuitive way.
		Here, the type of statistic is motivated by Lemmas \ref{lemma:dseparation} and \ref{lemma:probability}, and the value of interventions and threshold $t$ are motivated by Theorems \ref{thrm:NRBN_discrete_general} and \ref{thrm:NRBN_continuous_general}.
		See Appendix \ref{app:noisy_queries} (Algorithms \ref{algo:query_discrete} and \ref{algo:query_continuous}) for the specific details of the algorithms for discrete and continuous CBNs.
		\setcounter{algorithm}{1}
		\begin{algorithm}[ht]
			\caption{Noisy path query algorithm}
			\label{algo:query_informal}
			\begin{algorithmic}[1]
				\Require{Nodes $i$ and $j$, number of interventional samples $m$, and threshold $t$.}
				\Ensure{$\nQ(i,j)$}
				\State Intervene $\i$ by setting its value to $x_i \in Dom[\i]$, and observe $m$ samples of $\j$
				\State Compute a statistic of $\j$ and return $1$ if it is greater than $t$.
			\end{algorithmic}
		\end{algorithm}
		\paragraph{Discrete random variables.}
		\label{sec:discrete_local}
		
		In this paper we use conditional probability tables (CPTs) as the representation of the CPDs for discrete CBNs. 
		Next, we present a theorem that provides the sample complexity of a noisy path query.
		 
		\begin{theorem}
			\label{thrm:NRBN_discrete_general}
			Let $\B = (\Grm, \THT)$ be a discrete CBN, such that each random variable $\j$ has a finite domain $Dom[\j]$, with $\abs{Dom[\j]} \leq r$. Furthermore, let 
			\begin{multline}
			\gamma = \hspace{-0.05in} \min_{\substack{j \in \Vrm \\ i \in \pi_\Grm(j)}} \hspace{0.025in} \min_{\substack{x_i, x'_i \in Dom[\i] \\ \pbf(\j|do(\i = x_i)) \neq \pbf(\j|do(\i = x'_i))  }} \hspace{-0.1in} \nrm{ \pbf(\j|do(\i = x_i)) - \pbf(\j|do(\i = x'_i)) }_{\infty}, \nonumber
			\end{multline}
			and let $\hat{\Grm}=(\Vrm,\hat{\Erm})$ be the learned graph by using Algorithm \ref{algo:algo1}. Then for $\gamma > 0$ and a fixed probability of error $\delta \in (0,1)$, we have
			\(
			\Pr{\TR(\Grm) = \hat{\Grm}} \geq 1-\delta,
			\)
			provided that $m \in \bO{\frac{1}{\gamma^2} \left(  \ln n + \ln \frac{ r}{\delta} \right)}$ interventional samples are used per $\delta$-noisy partially-correct path query in Algorithm \ref{algo:query_discrete}.
		\end{theorem}
		
		Intuitively, the value $\gamma$ characterizes the \textit{minimum causal effect} among all the pair of parent-child nodes.
		Due to Assumption \ref{assump:faith}, and the fact that an edge represents a causal relationship, we have $\gamma > 0$.
		This value is used for deciding whether two empirical PMFs are equal or not in our path query algorithm (Algorithm \ref{algo:query_discrete}), which implements Claim 2 in Proposition \ref{lemma:probability}.
		Finally, in practice, the value of $\gamma$ is unknown\footnote{Several prior works from leading experts also have $\tilde{\mathcal{O}}(\frac{1}{\gamma^2})$ sample complexity for  an \textit{unknowable} constant $\gamma$. See for instance, \citep{Brenner13, obozinski2009high, ravikumar2011high}.}.
		Fortunately, knowing a lower bound of $\gamma$ suffices for structure recovery.

		\paragraph{Continuous random variables.}
		\label{sec:continuous_local}
		
		For continuous CBNs, our algorithm compares two empirical expected values for answering a path query. This is related to Claim 1 in Proposition \ref{lemma:probability}, since $\mathbb{E}[\j] \neq \mathbb{E}[\j|do(\i=x_i)]$ implies $P(\j) \neq P(\j|do(\i=x_i))$. We analyze continuous CBNs where every random variable is sub-Gaussian. The class of sub-Gaussian variates includes for instance Gaussian variables, any bounded random variable (e.g., uniform), any random variable with strictly log-concave density, and any finite mixture of sub-Gaussian variables. Note that sample complexity using sub-Gaussian variables has been studied in the past for other models, such as Markov random fields \citep{ravikumar2011high}. Next, we present a theorem that formally characterizes the class of continuous CBNs that our algorithm can learn, and provides the sample complexity for each noisy path query.
		\begin{theorem}
		\label{thrm:NRBN_continuous_general}
			Let $\B=(\Grm,\THT)$ be a continuous CBN such that each variable $\j$ is a sub-Gaussian random variable with full support on $\R$, with mean $\mu_j=0$ and variance $\sigma^2_j$.
			Let $\mu_{j|do(\i=z)}$ and $\sigma^2_{j|do(\i=z)}$ denote the expected value and variance of $\j$ after intervening $\i$ with value $z$, assuming also that the variables remain sub-Gaussian after performing an intervention. Furthermore, let 
			\begin{equation}
			\begin{split}
			\mu(\B,z) = \min_{\substack{j \in \Vrm, i \in \pi_\Grm(j)}} \abs{ \mu_{j|do(\i=z)} }, \quad \sigma^2(\B,z) = \maxx{ \max_{\substack{j \in \Vrm, i \in \pi_\Grm(j)}} \sigma^2_{j|do(\i=z)} }{ \max_{j \in \Vrm} \sigma^2_j }, \nonumber
			\end{split}
			\end{equation}
			and let $\hat{\Grm}=(\Vrm, \hat{\Erm})$ be the learned graph by using Algorithm \ref{algo:algo1}. If there exist an upper bound $\sigmasub$ and a finite value $z$ such that $\sigma^2(\B,z) \leq \sigmasub$ and $\mu(\B,z) \geq 1$, then for a fixed probability of error $\delta \in (0,1)$, we have
			\( 
			\Pr{\TR(\Grm) = \hat{\Grm}} \geq 1-\delta,
			\)
			provided that $m \in \bO{\sigmasub \log \frac{n}{\delta}}$ interventional samples are used per $\delta$-noisy partially-correct path query in Algorithm \ref{algo:query_continuous}.
		\end{theorem}
		Note that the conditions $\mu_j=0, \forall j \in V$, and $\mu(\mathcal{B},z) \geq 1$ are set to offer clarity in the derivations. One could for instance set an upper bound for the magnitude of $\mu_j$, assume $\mu(\mathcal{B},z)$  to be greater than this upper bound plus $1$, and still have the same sample complexity. Finally, our motivation for giving such conditions is that of guaranteeing a proper separation of the expected values in cases where there is effect of a variable $\i$ over another variable $\j$, versus cases where there is no effect at all. 
		
		Next, we define the additive sub-Gaussian noise model (ASGN).
		\begin{definition}
		\label{def:subgaussian}
			Let $\Grm=(\Vrm,\Erm)$ be a DAG, let $\W \in \R^{n\times n}$ be the matrix of edge weights and let $\S = \{\sigma^2_i \in \R_+ | i\in \Vrm\}$ be the set of noise variances. An additive sub-Gaussian noise network is a tuple $(\Grm, \P(\W,\S))$ where each variable $\i$ can be written as follows:
			\( \i = \sum_{j \in \pi_\Grm(i)} \Wrm_{ij}X_j + N_i,\ \forall i \in \Vrm,\)
			with $N_i$ being an independent sub-Gaussian noise with full support on $\R$, with zero mean and variance $\sigma_i^2$ for all $i \in \Vrm$, and $\Wrm_{ij} \neq 0$ iff $(j,i) \in \Erm$.
		\end{definition}

		\begin{remark}
		\label{remark:asgm}
			Let $\B=(\Grm, \P(\W,\S))$ be an ASGN network. We can rewrite the model in vector form as: $\xbf = \W \xbf + \nbf$ or equivalently $\xbf = (\I - \W)^{-1} \nbf$, where $\xbf=(X_1,\ldots,X_n)$ and $\nbf=(N_1,\ldots,N_n)$ are  the vector of random variables and the noise vector respectively. Additionally, we denote $\odot_i \W$ as the weight matrix $\W$ with its $i$-th row set to $0$. This means that we can interpret $\odot_i \W$ as the weight matrix after performing and intervention on node $i$ (mutilated graph). 
		\end{remark}
		We  now present a corollary that fulfills the conditions presented in Theorem \ref{thrm:NRBN_continuous_general}. 
		\begin{corollary}[Additive sub-Gaussian noise model]
		\label{corol:NRBN_subgaussian}
			Let $\B=(\Grm,\P(\W,\S))$ be an ASGN network as in Definition \ref{def:subgaussian}, such that $\noisej \leq \noisemax, \forall j \in \Vrm$. Also, let $\wmin = \min_{(i,j) \in \Erm} |\{(\I-\odot_i \W)^{-1}\}_{ji}|$, and $w_{max} = \max ( \lVert(\I-\W)^{-1}\rVert^2_{\infty,2}, \max_{i \in \Vrm} \lVert(\I-\odot_i\W)^{-1}\rVert^2_{\infty,2} )$. 
			If $z = 1/\wmin$ and $\sigmasub = \noisemax w_{max}$, then for a fixed probability of error $\delta \in (0,1)$, we have $P(\TR(\Grm) = \hat{\Grm}) \geq 1-\delta$. 
			Where $\hat{\Grm}=(\Vrm,\hat{\Erm})$ is the learned graph by using Algorithm \ref{algo:algo1}, and provided that  $m \in \bO{\sigmasub \log \frac{n}{\delta}}$ interventional samples are used per $\delta$-noisy partially-correct path query in Algorithm \ref{algo:query_continuous}.
		\end{corollary}

		The values of $\wmin$ and $\wmax$ follow the specifications of Theorem \ref{thrm:NRBN_continuous_general}. 
		In addition, the value of $\wmin$ is guaranteed to be greater than $0$ because of the faithfulness assumption (see Assumption \ref{assump:faith}). 
		For an example about our motivation to use the faithfulness assumption, see Appendix \ref{app:examples}.

\subsection{Recovery of Transitive Edges}
\label{sec:transitive}

In this section, we show a method to recover the transitive edges by using multiple-vertex interventions. This allows us to learn the full network. For this purpose, we present a new query defined as follows.

\begin{definition}[$\delta$-noisy transitive query]
	\label{def:transitive_query}
	Let $\Grm = (\Vrm, \Erm)$ be a DAG, and  let $\delta \in (0,1)$ be a probability of error. A $\delta$-noisy transitive query is a function $\nT_\Grm: \Vrm \times \Vrm \times 2^\Vrm \rightarrow \{0, 1\}$ such that $\nT_\Grm(i, j, \Srm) = 1$ with probability at least $1-\delta$ if $(i,j) \in \Erm$ is a transitive edge (where the additional path from $i$ to $j$ goes through $\Srm$), and $0$ otherwise. Here $\Srm \subseteq \pi_{\Grm}(j)$ is an auxiliary set necessary to answer the query, in order to block any influence from $i$ to $\Srm$, and to unveil the direct effect from $i$ to $j$.
\end{definition}

Algorithms \ref{algo:query_discrete_transitive} and $\ref{algo:query_continuous_transitive}$ (see Appendix \ref{app:noisy_transitive_queries}) show how to answer a transitive query for discrete and continuous CBNs respectively. Both algorithms are motivated on a property of CBNs, that is, $\forall i \in \Vrm$ and for every set $\Srm$ disjoint of $\{i, \pi_\Grm(i)\}$, we have $P(X_i | do(X_{\pi_\Grm(i)}=x_{\pi_\Grm(i)}),do(X_\Srm=x_\Srm)) = P(X_i | do(X_{\pi_\Grm(i)}=x_{\pi_\Grm(i)}))$. Thus, both algorithms intervene all the variables in $\Srm$, if $\Srm$ is the parent set of $j$, then $i$ will have no effect on $j$ and they return $0$, and $1$ otherwise. 

Recall that by using Algorithm \ref{algo:algo1} we obtain the transitive reduction of the CBN, thus, we have the true topological ordering of the CBN, and also for each node $i \in \Vrm$, we know its parent set or a subset of it. 
Using these observations, we can cleverly set the input $i$, $j$, and $\Srm$ of a noisy transitive query, as done in Algorithm \ref{algo:transitive_edges}. 
It is clear that Algorithm \ref{algo:transitive_edges} makes $\bO{n^2}$ noisy transitive queries in total. 
The time complexity to answer a transitive query for a discrete CBN is exponential in the maximum number of parents in the worst case. 
However, the sample complexity for queries in discrete and continuous CBNs remains polynomial in $n$ as prescribed in the following theorems.

\begin{algorithm}[ht]
	\caption{Learning the transitive edges by using noisy transitive queries}
	\label{algo:transitive_edges}
	\begin{algorithmic}[1]
		\Require{Transitively reduced DAG $\hat{\Grm}=(\Vrm, \hat{\Erm})$ (output of Algorithm 1)} 
		\Ensure{DAG $\tilde{\Grm}=(\Vrm,\tilde{\Erm})$}
		\State $\Psi \gets \texttt{TopologicalOrder}(\hat{\Grm})$;\ \ \ $\hat{\pi}(i) \gets \{u \in \Vrm  | (u,i) \in \hat{\Erm} \}$ (current parents of $i$);\ \ \ $\tilde{\Erm} \gets \hat{\Erm}$
		\For{$j=2 \ldots n$}
			\For{$i=j-1,j-2, \ldots 1$}
				\State \textbf{if}  $\nT(\Psi_i, \Psi_j, \hat{\pi}(\Psi_j)) = 1$ \textbf{then} $\tilde{\Erm} \gets \tilde{\Erm} \cup \{(\Psi_i, \Psi_j)\}$ and $\hat{\pi}(\Psi_j) \gets \hat{\pi}(\Psi_j) \cup \Psi_i$
			\EndFor
		\EndFor
	\end{algorithmic}
\end{algorithm}

\begin{theorem}
	\label{thrm:discrete_transitive}
	Let $\B = (\Grm,\THT)$ be a discrete CBN, such that each random variable $\j$ has a finite domain $Dom[\j]$, with $\abs{Dom[\j]} \leq r$. Furthermore, let 
	\begin{multline}
	\hspace{-0.08in}	\gamma = \hspace{-0.1in} \min_{\substack{j \in \Vrm \\ \Srm \subseteq \pi_\Grm(j), |S| \geq 1}} \hspace{0.025in} \min_{\substack{x_\Srm, x'_\Srm \in \times_{i \in \Srm} Dom[\i]  \\ \pbf(\j|do(X_\Srm = x_\Srm)) \neq \pbf(\j|do(X_\Srm = x'_\Srm))  }} \hspace{-0.25in} \nrm{ \pbf(\j|do(X_\Srm = x_\Srm)) - \pbf(\j|do(X_\Srm = x'_\Srm)) }_{\infty}, \nonumber
	\end{multline}
	and let $\tilde{\Grm}=(\Vrm,\tilde{\Erm})$ be the output of Algorithm \ref{algo:transitive_edges}. Then for $\gamma > 0$ and a fixed probability of error $\delta \in (0,1)$, we have
	\(	\Pr{\Grm = \tilde{\Grm}} \geq 1-\delta, \)
	provided that $m \in \bO{\frac{1}{\gamma^2} \left(  \ln n + \ln \frac{ r}{\delta} \right)}$ interventional samples are used per $\delta$-noisy transitive query in Algorithm \ref{algo:query_discrete_transitive}.
\end{theorem}

\begin{theorem}
	\label{thrm:continuous_transitive}
	Let $\B=(\Grm,\THT)$ be a continuous CBN such that each variable $\j$ is a sub-Gaussian random variable with full support on $\R$, with mean $\mu_j=0$ and variance $\sigma^2_j$. Let $\mu_{j|do(X_\Srm=\mathbf{1}z)}$ and $\sigma^2_{j|do(X_\Srm=\mathbf{1}z)}$ denote the expected value and variance of $\j$ after intervening each node of  $X_\Srm$ with value $z$. Furthermore, let 
	\begin{equation}
	\begin{split}
		\mu(\B,z_1,z_2) &= \min_{\substack{j \in \Vrm, \Srm \subseteq \pi_\Grm(j), |\Srm|\geq 2, i\in \Srm} } \abs{ \mu_{j|do(X_{\Srm-\{i\}}=\mathbf{1}z_1, X_i=z_2)} }, \\
		\sigma^2(\B,z_1,z_2) &= \max \Bigg(  \max_{j \in \Vrm} \sigma^2_j,  
		  \max_{\substack{j \in \Vrm, \Srm \subseteq \pi_\Grm(j), |\Srm|\geq 2,  i\in \Srm}}  \sigma^2_{j|do(X_{\Srm-\{i\}}=\mathbf{1}z_1, X_i=z_2)} \Bigg), \nonumber
	\end{split}
	\end{equation}
	and let $\tilde{G}=(\Vrm,\tilde{\Erm})$ be the output of Algorithm \ref{algo:transitive_edges}. If there exist an upper bound $\sigmasub$ and finite values $z_1,z_2$ such that $\sigma^2(\B,z_1,z_2) \leq \sigmasub$ and $\mu(\B,z_1,z_2) \geq 1$, then for a fixed probability of error $\delta \in (0,1)$, we have
	\( \Pr{\Grm = \tilde{\Grm}} \geq 1-\delta, \)
	provided that $m \in \bO{\sigmasub \log \frac{n}{\delta}}$ interventional samples are used per $\delta$-noisy transitive query in Algorithm \ref{algo:query_continuous_transitive}.
\end{theorem}

Next, we show that ASGN networks can fulfill the conditions in Theorem \ref{thrm:continuous_transitive}.

\begin{corollary}
	\label{corol:NRBN_subgaussian_transitive}
	Let $\B=(\Grm,\P(\W,\S))$, and $\noisemax$ follow the same definition as in Corollary \ref{corol:NRBN_subgaussian}. 
	Let $\wmin = \min_{ij} |\Wrm_{ij}|$, and $\wmax = \max( \lVert(\I-\W)^{-1}\rVert^2_{\infty,2}, \max_{\substack{j \in \Vrm, \Srm \subseteq \pi_\Grm(j)}} \lVert(\I-\odot_\Srm \W)^{-1}\rVert^2_{\infty,2})$. 
	If $z_1=0, z_2= 1/\wmin$, and $\sigmasub = \noisemax w_{max}$, then for a fixed probability of error $\delta \in (0,1)$, we have $P(\Grm = \tilde{\Grm}) \geq 1-\delta$, provided that  $m \in \bO{\sigmasub \log \frac{n}{\delta}}$ interventional samples are used per $\delta$-noisy transitive query in Algorithm \ref{algo:query_continuous_transitive}.
\end{corollary}

	\vspace{-0.1in}
	\section{Extensions}
	\label{sec:extensions}
		\paragraph{Learning rooted trees.}
			Here we make use of the results in \cite{wang2016reconstructing}, for rooted trees of node degree at most $d$.
			Theorem 4 in \cite{wang2016reconstructing} states that for a fixed probability error $\delta \in (0,1)$, one can reconstruct a rooted tree with probability $1-\delta$ in $\bO{\frac{1}{\delta} \frac{1}{(1/2 - \epsilon )^2} dn \log^2n\log\frac{dn}{\delta}}$ time provided that a total of $\bO{\frac{1}{(1/2-\epsilon)^2} n \log\frac{dn}{\delta}}$ noisy path queries are used, where $\epsilon$ relates to the confidence of the noisy path query.
			The number of queries is improved with respect to the $n^2$ queries used for general DAGs in the previous section.
			Finally, recall that in the previous section we made use of \textit{partially-correct} path queries, for this part we require a stronger version of noisy path query, which is defined below.
			
			\begin{definition}[$\epsilon$-noisy path query]
			\label{def:noisy_path_query}
			Let $\Grm = (\Vrm, \Erm)$ be a DAG, and let $Q_\Grm$ be a path query. Let $\epsilon \in (0,1/2)$ be a probability of error. A $\epsilon$-noisy path query is a function $\tilde{Q}_\Grm: \Vrm  \times \Vrm \rightarrow \{0, 1\}$ such that $\tilde{Q}_\Grm(i, j) = Q_\Grm(i,j)$ with probability at least $1-\epsilon$, and $\tilde{Q}_\Grm(i, j) = 1 - Q_\Grm(i,j)$ with probability at most $\epsilon$.
			\end{definition}
			The following states the sample complexity for exact learning of rooted trees in the discrete case.
			\begin{proposition}
				\label{prop:trees}
				Let $\B = (\Grm, \THT)$ be a discrete CBN, such that each random variable $\j$ has a finite domain $Dom[\j]$, with $\abs{Dom[\j]} \leq r$. Furthermore, let 
				\begin{multline}
				\gamma = \hspace{-0.05in} \min_{\substack{j \in \Vrm \\ i \in \Vrm}} \hspace{0.025in} \min_{\substack{x_i, x'_i \in Dom[\i] \\ \pbf(\j|do(\i = x_i)) \neq \pbf(\j|do(\i = x'_i))  }} \hspace{-0.1in} \nrm{ \pbf(\j|do(\i = x_i)) - \pbf(\j|do(\i = x'_i)) }_{\infty}, \nonumber
				\end{multline}
				and let $\hat{\Grm}=(\Vrm,\hat{\Erm})$ be the learned graph by using Algorithm 7 in \cite{wang2016reconstructing}. Then for $\gamma > 0$ and a fixed probability of error $\delta \in (0,1)$, we have
				\(
				\Pr{\Grm = \hat{\Grm}} \geq 1-\delta,
				\)
				provided that $m \in \bO{\frac{1}{\gamma^2} \left(  \ln n + \ln \frac{ r}{\delta} \right)}$ interventional samples are used per $\delta$-noisy path query in Algorithm \ref{algo:query_discrete}.
			\end{proposition}
			We use the same Algorithm \ref{algo:query_discrete} to answer a $\epsilon$-noisy path query.
			The difference is that now $\gamma$ represents the \textit{minimum causal effect} among all pair nodes and not only parent-child nodes.

	\paragraph{On Imperfect  Interventions.}
	\label{sec:imperfect}
		Here we state some results on imperfect interventions. In Appendix \ref{app:imperfect_interventions}, we show that the sample complexity for discrete CBNs is scaled by $\alpha^{-1}$, where $\alpha$ accounts for the degree of uncertainty in the intervention. While for CBNs of sub-Gaussian random variables, the sample complexity still has the same dependence on an upper bound of the variances.

	\section{Experiments}
	\label{sec:Experiments}
	
		In Appendix \ref{appendix:experiments_synthetic}, we tested our algorithms for perfect and imperfect interventions in  synthetic networks, in order to empirically show the logarithmic phase transition of the number of interventional samples (see Figure \ref{fig:experiments_main} as an example). 
		Appendix \ref{appendix_proportion} shows that in several benchmark BNs, most of the graph belongs to its transitive reduction, meaning that one can learn most of the network in polynomial time. 
		Appendix \ref{appendix_experiments} shows experiments on some of these benchmark networks, using the aforementioned algorithms and also our algorithm for learning transitive edges, thus recovering the full networks.
		Finally, in Appendix \ref{appendix_nature}, as an illustration of the availability of interventional data, we show experimental evidence using three gene perturbation datasets from \cite{xiao2015gene,harbison2004transcriptional}.
		\begin{figure}[!ht]
			\centering
			\begin{minipage}[c]{0.41\linewidth}
				\centering
				\includegraphics[width=\linewidth]{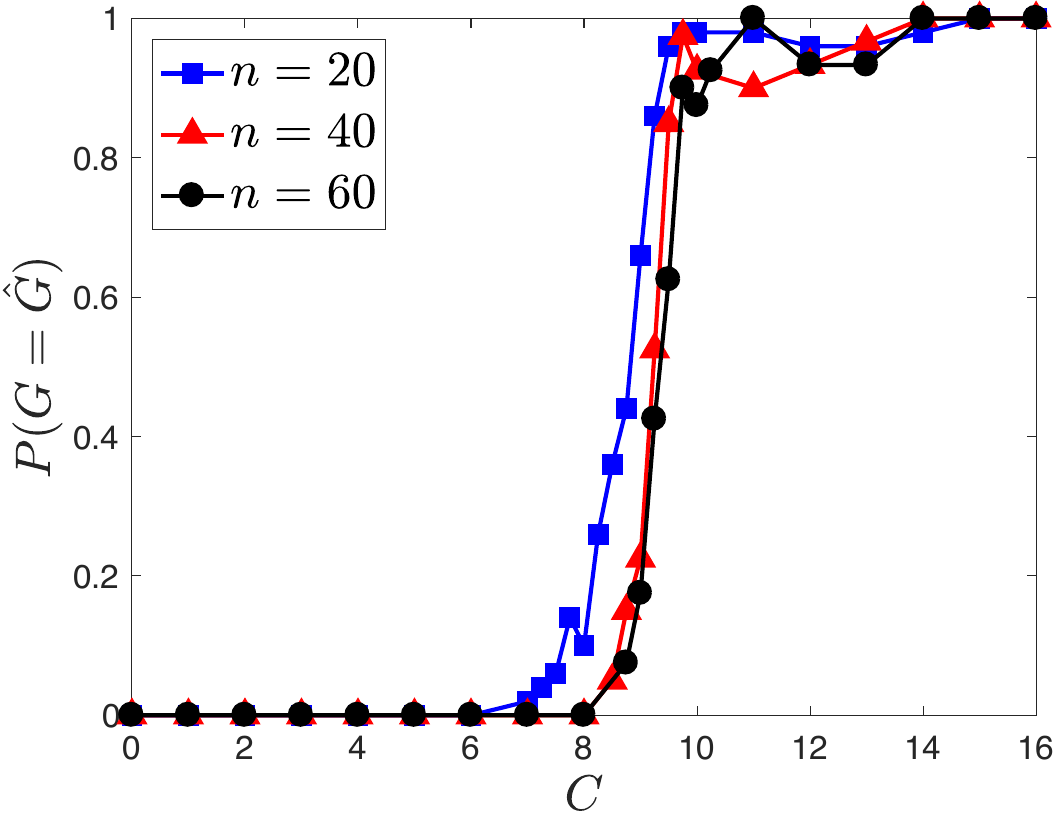}
			\end{minipage}
			\hspace{0.5in}
			\begin{minipage}[c]{0.41\linewidth}
				\centering
				\includegraphics[width=\linewidth]{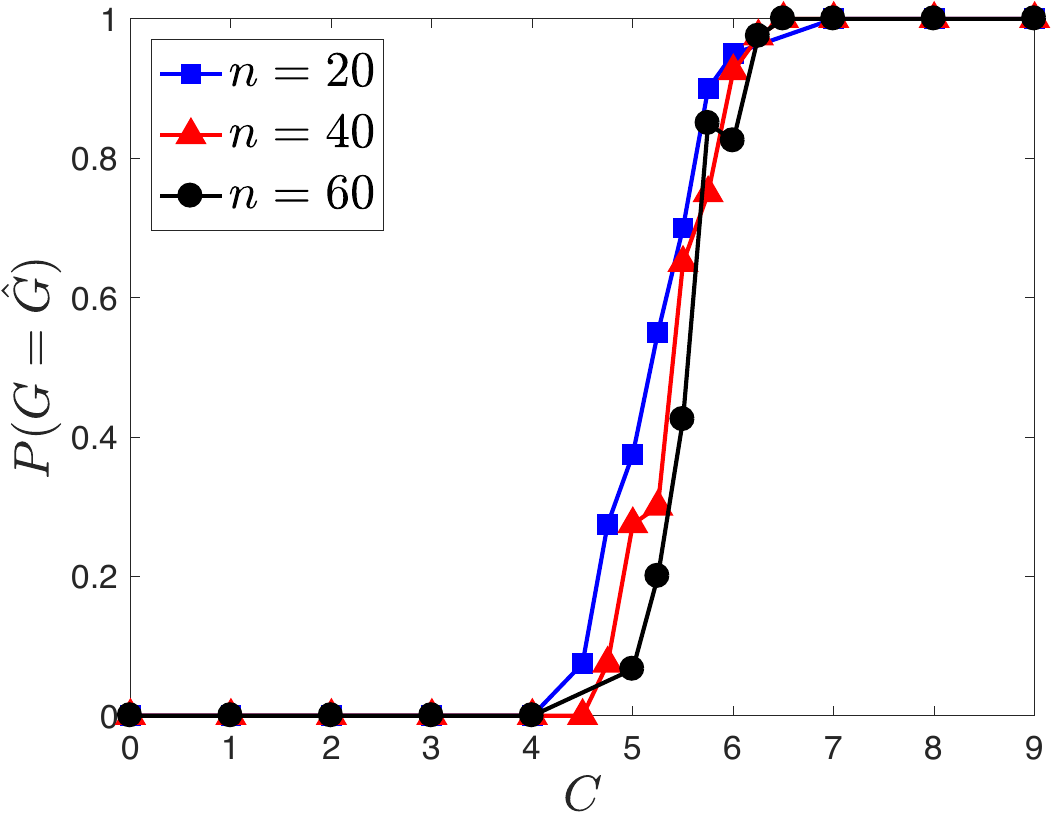}  
			\end{minipage}
			\caption{
			{\small
			(Left) Probability of correct structure recovery of the transitive reduction of a discrete CBN vs. number of samples per query, where the latter was set to $e^C\log nr$, with all CBNs having $r=5$ and $\gamma \geq 0.01$. 
			(Right) Similarly, for continuous CBNs, the number of samples per query was set to $e^C\log n$, with all CBNs having $\lVert(\I-\W)^{-1}\rVert^2_{2,\infty} \leq 20$. 
			Finally, we observe that there is a sharp phase transition from recovery failure to success in all cases, and the $\log n$ scaling holds in practice, as prescribed by Theorems \ref{thrm:NRBN_discrete_general}, \ref{thrm:NRBN_continuous_general}.}
			}
			\label{fig:experiments_main}
		\end{figure}

	\section{Future Work}
	\label{sec:future}
	There are several ways of extending this work. For instance, it would be interesting to analyze other classes of interventions with uncertainty, as in \cite{eaton2007exact}. 
	For continuous CBNs, we opted to use expected values and not to compare continuous distributions directly. 
	The fact that the conditioning is with respect to a continuous random variable makes this task more complex than the typical comparison of continuous distributions. 
	Still, it would be interesting to see whether kernel density estimators \cite{liu2012exponential} could be beneficial.

	\bibliographystyle{icml2018}
	\bibliography{paper}
	
	\normalsize
	\clearpage
	\begin{appendices}
	\onecolumn

\def\toptitlebar{
	\hrule height4pt
	\vskip .25in}

\def\bottomtitlebar{
	\vskip .25in
	\hrule height1pt
	\vskip .25in}

\thispagestyle{empty}
\hsize\textwidth
\linewidth\hsize \toptitlebar {\centering
{\large\bf SUPPLEMENTARY MATERIAL \\ Computationally and statistically efficient learning of causal Bayes nets using path queries \par}}
\vspace{-0.1in} \bottomtitlebar

\section{Discussion}
\label{app:discussion}

	\paragraph{Learning causal Bayes nets from purely interventional data.}
	
	Our interest in purely interventional data stems from our goal of discovering the true causal relationships.
	We perform single-vertex interventions for each node, which agrees with the numbers of single-vertex interventions sufficient and in the worst-case necessary to identify any DAG, as shown in \cite{Eberhardt05onthe}.
	
	\paragraph{Availability of purely interventional data.}
	
	The availability of purely interventional data is an implicit assumption in several prior works, which equivalently assume that one can perform an intervention on any node \citep{murphy2001active, tong2001active, he2008active, hauser2012two, shanmugam2015learning, kocaoglu2017experimental}.
	As an illustration of the availability of interventional data, as well as the applicability of our method, we show experimental evidence using three gene perturbation datasets from \citep{xiao2015gene, harbison2004transcriptional}. (See Appendix \ref{appendix_nature}.) 

\section{Algorithms}
\label{app:algorithms}

	\subsection{Algorithm for Transitive Reduction}
	\label{app:algo_transitive_reduction}

		As proved in  \cite{aho1972transitive}, the time complexity of the best algorithm for finding the  transitive reduction of a DAG is the same as the time to compute the transitive closure of a graph or to perform Boolean matrix multiplication. 
		Therefore, we can use any exact algorithm for fast matrix multiplication, such as \cite{le2014powers}, which has $\bO{n^{2.3729}}$ time complexity. 
		As a result, the time complexity of Algorithm \ref{algo:learning_tbn} is dominated by the computation of the transitive reduction since answering a query $\nQ(i,j)$ is in $\tilde{\mathcal{O}}(\log n)$. Finally, note that performing $n^2$ queries (one per each node pair) is equivalent to performing $n$ single-vertex interventions, in which we intervene one node and observe the remaining $n-1$ nodes. This number of interventions is necessary in the worst case, as discussed in \cite{Eberhardt05onthe}.
		
		\begin{algorithm}[ht]
		\caption{Learning the transitive reduction by using noisy path queries}
		\label{algo:learning_tbn}
		\begin{algorithmic}[1]
		\Require{  Vertex set $\Vrm$}
		\Ensure{  Edge set $\hat{\Erm}$}
		\State $\hat{\Erm} \gets \varnothing$
		\For{$i=1 \ldots n$}
			\For{$j=1 \ldots n$}
				\If{$i \neq j$ and $\nQ(i,j) = 1$}
					\State $\hat{\Erm} \gets \hat{\Erm} \cup \{(i,j)\}$
				\EndIf
			\EndFor
		\EndFor
		\State $\hat{\Erm} \gets \TR(\hat{\Erm})$ 
		\end{algorithmic}
		\end{algorithm}

		Assuming that we have correct answers for all path queries, Algorithm \ref{algo:learning_tbn} will indeed exactly recover the $\TR(\Grm)$ of any DAG $\Grm$. However, this is not necessary. We can recover the true transitive reduction, $\TR(\Grm)$, if we have correct answers for queries $\Q(i,j)$ when $i \in \pi_\Grm(j)$, and when there is no directed path from $i$ to $j$, and arbitrary answers when there is a directed path from $i$ to $j$.
		This is because the \emph{transitive reduction} step will remove every transitive edge. It is the previous observation that motivated our characterization of noisy queries given in Definition \ref{def:noisy_query_1}. 

	\subsection{Noisy Path Query Algorithms}
	\label{app:noisy_queries}
	
		Algorithms \ref{algo:query_discrete} and \ref{algo:query_continuous} present our algorithms for answering a noisy path query $\tilde{Q}(i,j)$ motivated by Theorems \ref{thrm:NRBN_discrete_general} and \ref{thrm:NRBN_continuous_general} respectively. For discrete CBNs, we first create a list $\Li$ of size $d = |Dom[\i]|$, containing the empirical probability mass functions (PMFs) of $\j$ after intervening $\i$ with all the possible values from its domain $Dom[\i]$. Next, if the $\linf$-norm of the difference of any pair of PMFs in $\Li$ is greater than a constant $\gamma$, then we answer the query with $1$, and $0$ otherwise. For continuous CBNs, we intervene $\i$ with a constant value $z$ and compute the empirical expected value of $\j$. We then output $1$ if the absolute value of the expected value is greater than $\nicefrac{1}{2}$, and $0$ otherwise.
		(The threshold of $\nicefrac{1}{2}$ is due to the particular way to set $z$, as prescribed by Theorem \ref{thrm:NRBN_continuous_general} and Corollary \ref{corol:NRBN_subgaussian}.)
		
		\begin{algorithm}[ht]
			\caption{Noisy path query algorithm for discrete variables}
			\label{algo:query_discrete}
			\begin{algorithmic}[1]
				\Require{Nodes $i$ and $j$, number of interventional samples $m$, and constant $\gamma$.}
				\Ensure{$\nQ(i,j)$}
				\State $\Li \gets$ emptyList()
				\For{$x_i \in Dom[\i]$}
				\State Intervene $\i$ by setting its value to $x_i$, and obtain $m$ samples $x\si{1}_j,\ldots,x\si{m}_j$ of $\j$
				\State $\hat{p}_k = \frac{1}{m} \sum_{l=1}^m \indicator[x^{(l)}_j = k], \forall k \in \hspace{-0.01in} Dom[X_j]$
				\State Add $\hat{\pbf}$ to the list $\Li$
				\EndFor
				\State $\nQ(i,j) \gets \indicator [( \exists\ \hat{\pbf}, \hat{\qbf} \in \Li)\ \lVert \hat{\pbf}-\hat{\qbf} \rVert_\infty > \gamma ]$
			\end{algorithmic}
		\end{algorithm}
		
		\begin{algorithm}[ht]
			\caption{Noisy path query algorithm for continuous variables}
			\label{algo:query_continuous}
			\begin{algorithmic}[1]
				\Require{Nodes $i$ and $j$, number of interventional samples $m$, and constant $z$ \redcolor{(set as prescribed by Theorem \ref{thrm:NRBN_continuous_general} or Corollary \ref{corol:NRBN_subgaussian}.)}}
				\Ensure{$\nQ(i,j)$}
				\State Intervene $\i$ by setting its value to $z$, and obtain $m$ samples $x_j^{(1)},\ldots,x_j^{(m)}$ of $\j$
				\State $\hmu \gets \frac{1}{m} \sum_{k=1}^m x_j^{(k)}$
				\State $\nQ(i,j) \gets \indicator[|\hmu| > \nicefrac{1}{2}]$ \Comment{\redcolor{(The threshold of \nicefrac{1}{2} is due to the particular way to set $z$, as prescribed by Theorem \ref{thrm:NRBN_continuous_general} and Corollary \ref{corol:NRBN_subgaussian}.)}}
			\end{algorithmic}
		\end{algorithm}

	\subsection{Noisy Transitive Query Algorithms}
	\label{app:noisy_transitive_queries}
	
		Algorithms \ref{algo:query_discrete_transitive} and $\ref{algo:query_continuous_transitive}$ show how to answer a transitive query for discrete and continuous CBNs respectively. Both algorithms are motivated on a property of CBNs, that is, $\forall i \in \Vrm$ and for every set $\Srm$ disjoint of $\{i, \pi_\Grm(i)\}$, we have $P(X_i | do(X_{\pi_\Grm(i)}=x_{\pi_\Grm(i)}),do(X_\Srm=x_\Srm)) = P(X_i | do(X_{\pi_\Grm(i)}=x_{\pi_\Grm(i)}))$. Thus, both algorithms intervene all the variables in $\Srm$, if $\Srm$ is the parent set of $j$, then $i$ will have no effect on $j$ and they return $0$, and $1$ otherwise. 
		
		\begin{algorithm}[ht]
			\caption{Noisy transitive query algorithm for discrete variables}
			\label{algo:query_discrete_transitive}
			\begin{algorithmic}[1]
				\Require{Nodes $i$ and $j$, set of nodes $\Srm$, number of interventional samples $m$, and constant $\gamma$.}
				\Ensure{$\nT(i,j,\Srm)$}
				\State $\Li \gets$ emptyList()
				\For{$x_\srm \in \times_{k\in \Srm} Dom[X_k]$}
				\State Intervene set $X_\Srm$ by setting its value to $x_\srm$
				\For{$x_i \in Dom[\i]$}
				\State Intervene $\i$ by setting its value to $x_i$, and obtain $m$ samples $x\si{1}_j,\ldots,x\si{m}_j$ of $\j$
				\State $\hat{p}_k = \frac{1}{m} \sum_{l=1}^m \indicator[x^{(l)}_j = k], \forall k \in \hspace{-0.01in} Dom[X_j]$
				\State Add $\hat{\pbf}$ to the list $\Li$
				\EndFor
				\State $\nT(i,j,\Srm) \gets \indicator [( \exists\ \hat{\pbf}, \hat{\qbf} \in \Li)\ \lVert \hat{\pbf}-\hat{\qbf} \rVert_\infty > \gamma ]$
				\If{$\nT(i,j,\Srm) = 1$}
			 	STOP
				\EndIf
				\EndFor
			\end{algorithmic}
		\end{algorithm}

		\begin{algorithm}[ht]
			\caption{Noisy transitive query algorithm for continuous variables}
			\label{algo:query_continuous_transitive}
			\begin{algorithmic}[1]
				\Require{Nodes $i$ and $j$, set of nodes $\Srm$, number of interventional samples $m$, and constants $z_1, z_2$ \redcolor{(set as prescribed by Theorem \ref{thrm:continuous_transitive} or Corollary \ref{corol:NRBN_subgaussian_transitive}.)}}
				\Ensure{$\nT(i,j,\Srm)$}
				\State Intervene all variables $X_\Srm$ by setting their values to  $z_1$
				\State Intervene $\i$ by setting its value to $z_2$, and obtain $m$ samples $x_j^{(1)},\ldots,x_j^{(m)}$ of $\j$
				\State $\hmu \gets \frac{1}{m} \sum_{k=1}^m x_j^{(k)}$
				\State $\nT(i,j,\Srm) \gets \indicator[|\hmu| > \nicefrac{1}{2}]$ \Comment{\redcolor{(The threshold of \nicefrac{1}{2} is due to the particular way to set $z_1$ and $z_2$, as prescribed by Theorem \ref{thrm:continuous_transitive} and Corollary \ref{corol:NRBN_subgaussian_transitive}.)}}
			\end{algorithmic}
		\end{algorithm}

	\subsection{Query Algorithm for Discrete Networks Under Imperfect Interventions}
	\label{app:query_imperfect}
	
		Algorithm \ref{algo:query_imperfect} shows how to answer a noisy query for discrete CBNs under imperfect interventions.
		
		\begin{algorithm}[ht]
			\caption{Noisy path query algorithm for discrete variables under imperfect interventions.}
			\label{algo:query_imperfect}
			\begin{algorithmic}[1]
				\Require{Nodes $i$ and $j$, number of interventional samples $m$, and constant $\gamma$}
				\Ensure{$\nQ(i,j)$}
				\State $\Li \gets$ emptyList()
				\For{$x_i \in Dom[\i]$}
				\State Try to intervene $\i$ with value $x_i$, and obtain $m$ pair samples $(x\si{1}_i, x\si{1}_j),\ldots,(x\si{m}_i, x\si{m}_j)$ of $\i$ and $\j$
				\State $\hat{p}_k = \frac{1}{\sum_{l=1}^m \indicator[x^{(l)}_i = x_i]} \sum_{l=1}^m \indicator[x^{(l)}_j = k \wedge x^{(l)}_i = x_i], \forall k \in Dom[X_j]$
				\State Add $\hat{\pbf}$ to the list $\Li$
				\EndFor
				\State $\nQ(i,j) \gets \indicator [( \exists\ \hat{\pbf}, \hat{\qbf} \in \Li)\ \lVert \hat{\pbf}-\hat{\qbf} \rVert_\infty > \gamma ]$
			\end{algorithmic}
		\end{algorithm}

\section{On Imperfect Interventions}
\label{app:imperfect_interventions}

	In this section we relax the assumption of perfect interventions and analyze the sample complexity of a noisy path query.  \cite{eaton2007exact} analyzed a general framework of interventions named as \emph{uncertain interventions}.
	In general terms, we model an imperfect intervention by adding some degree of uncertainty to the intervened variable. Note that the main distinction with respect to perfect interventions is that now the intervened variable is a random variable, meanwhile in perfect interventions the intervened variable is considered a constant.
	
	\paragraph{Discrete random variables.}
	For a discrete CBN, we assume that an intervention follows a Bernoulli trial. That is, when one wants to intervene a variable $\i$ with target value $v$, the probability that $\i$ takes the target value $v$ is $\phi_i$, i.e., $P(\i = v) = \phi_i$, and $P(\i \neq v) = 1-\phi_i$ otherwise.
	
	To answer a noisy path query under this setting, we modify lines $3$ and $4$ of Algorithm \ref{algo:query_discrete}. In line $3$, we now get pair samples $\{(x_{i}^{(1)},x_j^{(1)}), \ldots, (x_{i}^{(m)},x_j^{(m)})\}$. In line $4$, we know estimate $\pbf(\j|do(\i=x_i))$ as follows:  $\forall k \in Dom[X_j], \hat{p}_k = \frac{1}{\sum_{l=1}^m \indicator[x^{(l)}_i = x_i]} \sum_{l=1}^m \indicator[x^{(l)}_j = k \wedge x^{(l)}_i = x_i]$. For completeness, we include the algorithm in Appendix \ref{app:query_imperfect}. Finally, the number of interventional samples $m$ is prescribed by the following theorem.
	
	\begin{theorem}
		\label{thrm:imperfect_discrete}
		Let $\B=(\Grm,\THT)$, $r$, and $\gamma$ follow the same definition as in Theorem \ref{thrm:NRBN_discrete_general}. Let $\alpha$ be a constant such that for all $i \in \Vrm$, $1/2 \leq \alpha \leq \phi_i$, in terms of imperfect interventions.
		Let $\hat{\Grm}=(\Vrm,\hat{\Erm})$ be the output of Algorithm \ref{algo:learning_tbn}. Then for $\gamma > 0$ and a fixed probability of error $\delta \in (0,1)$, we have
		$	P(\TR(\Grm) = \hat{\Grm}) \geq 1-\delta, $
		provided that $m \in \bO{\frac{1}{\alpha\gamma^2} \left( \ln n + \ln \frac{ r}{\delta} \right)}$ interventional samples are used per $\delta$-noisy partially-correct path query in the \emph{modified} Algorithm \ref{algo:query_discrete} as described above.
	\end{theorem}
	
	In practice, knowing the value of each $\phi_i$ can be hard to obtain, hence our motivation to introduce a lower bound $\alpha$ in Theorem \ref{thrm:imperfect_discrete}.

	\paragraph{Continuous random variables.}
	For continuous CBNs, we model an imperfect intervention by assuming that the intervened variable is also a sub-Gaussian variable. That is, when one intervenes a variable $\i$ with target value $v$, $\i$ becomes a sub-Gaussian variable with mean $v$ and variance $\nu^2_i$. Finally, we continue using Algorithm \ref{algo:query_continuous} to answer noisy path queries under this new setting.
	
	\begin{theorem}
		\label{thrm:imperfect_continuous}
		Let $\B=(\Grm,\THT), \mu_j=0$,  and $\sigma^2_j$ follow the same definition as in Theorem \ref{thrm:NRBN_continuous_general}. Let $\mu_{j|do(\i=z)}$ and $\sigma^2_{j|do(\i=z)}$ denote the expected value and variance of $\j$ after perfectly intervening $\i$ with value $z$. Furthermore, let $\mu(\B,z) = \min_{(i,j)\in \Erm} | \mathbb{E}_{\i}[\mu_{j|do(\i=z)}] |$, and $\sigma^2(\B,z) = \max( \max_{(i,j)\in \Erm}  \mathbb{E}_{\i}[\sigma^2_{j|do(\i=z)}], \max_{j \in \Vrm} \sigma^2_j ).$
		Let $\hat{\Grm}=(\Vrm,\hat{\Erm})$ be the output of Algorithm \ref{algo:learning_tbn}. If there exist an upper bound $\sigmasub$ and a finite value $z$ such that $\sigma^2(\B,z) \leq \sigmasub$ and $\mu(\B,z) \geq 1$, then for a fixed probability of error $\delta \in (0,1)$, we have
		$
			P(\TR(\Grm) = \hat{\Grm}) \geq 1-\delta,
		$
		provided that $m \in \bO{\sigmasub \log \frac{n}{\delta}}$ interventional samples are used per $\delta$-noisy partially-correct path query in Algorithm \ref{algo:query_continuous}.
	
	\end{theorem}
	
	The motivation of the conditions in Theorem \ref{thrm:imperfect_continuous} are similar to Theorem \ref{thrm:NRBN_continuous_general}. Next, we show that ASGN models can fulfill the conditions above.
	
	\begin{corollary}
		\label{corollary:imperfect_NRBN_subgaussian}
		Under the settings given in Corollary \ref{corol:NRBN_subgaussian}. If for all $j \in \Vrm$,  $\nu_j^2 \leq \noisemax$  in terms of imperfect interventions. Then, for a fixed probability of error $\delta \in (0,1)$, we have $P(\TR(\Grm) = \hat{\Grm}) \geq 1-\delta$ provided that $m \in \bO{ \sigmasub \log \frac{n}{\delta}}$ interventional samples are used per $\delta$-noisy partially-correct path query in Algorithm \ref{algo:query_continuous}.
	\end{corollary}

\section{Examples}
\label{app:examples}
	
	\subsection{Example for the use of faithfulness assumption}
		Consider the following ASGN network in Figure \ref{fig:asgn_faith}, assume that $X_1$ is intervened, then we have that the expected value of $X_3$ is $0$ regardless of the value of the intervention. This occurs because the effect is canceled via the directed paths $\{(1,2),(2,3)\}$ and $\{(1,3)\}$. 
		This motivated us to use the faithfulness assumption and rule out such ``pathological'' parameterizations. 
		Finally, in practice, the values of $\wmin$ and $\sigmasub$ are unknown. Fortunately, knowing a lower bound of $\wmin$ and an upper bound of $\sigmasub$ suffices for structure recovery.
		
		\begin{figure}[ht]
			\begin{center}
				\begin{tikzpicture}[scale=1.0]
				\tikzset{>=latex}
				\tikzstyle{vertex}=[circle, fill=white, draw, inner sep=0pt, minimum size=15pt]
				\node[vertex][label=center:$1$](x0) at (1.5,2) {};
				\node[vertex][label=center:$2$](x1) at (1,1) {};
				\node[vertex][label=center:$3$](x2) at (2,1) {};
				\tikzset{edge/.style={->,> = latex'}}
				\draw[edge] (x0) -- (x1) node[midway, above, sloped] {\tiny $+1$};		
				\draw[edge] (x1) -- (x2) node[midway, below, sloped] {\tiny $+1$};		
				\draw[edge] (x0) -- (x2) node[midway, above, sloped] {\tiny $-1$};		
				\end{tikzpicture}
			\end{center}
			\caption{An ASGN network in which the effect of $X_1$ on $X_3$ is none.}
			\label{fig:asgn_faith}
		\end{figure}
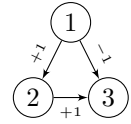
		
	\subsection{Example about the Number of Experiments in the Worst Case for Multiple-Vertex Interventions}
	
		Consider the following DAG in Figure \ref{fig:example_hard} of 6 binary variables.
		Such that, $P(A) = 0.5, P(B)=0.2, P(C)=0.1$.
		Let us also assume that $P(X_1|\neg A \neg B \neg C) = 0.4$, and for any other combination of $A,B,C$ we have $P(X_1|\ \cdot \ ) = 0.8$.
		
		Let us say that we perform a multiple-vertex intervention of $A,B,C$, and that we want to unveil the causal edge $(A,X_1)$.
		For this DAG we have that $P(X_1|do(A),do(B),do(C)) = P(X_1|ABC)$.
		Next let us say that we randomly select the configuration $A,\neg B, C$ for the intervention.
		Then $P(X_1|do(A\neg B C)) = P(X_1|A\neg BC) = 0.8$, in order to discover the causal edge, we also perform the following intervention, $P(X_1|do(\neg A\neg B C)) = P(X_1|\neg A\neg BC) = 0.8$.
		Which results in an ``independence'' or apparent no causal effect.
		In order to unveil the causal edge $(A,X_1)$, it is required to intervene with the configurations $A,\neg B, \neg C$ and $\neg A,\neg B, \neg C$, which in the worst case may be a single configuration out of an exponential number of possible configurations that allows to find the \textit{direct} causal effect.
		
		\begin{figure}[ht]
			\begin{center}
				\begin{tikzpicture}[scale=1.0]
				\tikzset{>=latex}
				\tikzstyle{vertex}=[circle, fill=white, draw, inner sep=0pt, minimum size=15pt]
				\node[vertex][label=center:$A$](A) at (1,2) {};
				\node[vertex][label=center:$B$](B) at (-1,1) {};
				\node[vertex][label=center:$C$](C) at (3,1) {};
				\node[vertex][label=center:$X_1$](x1) at (-0.5,0) {};
				\node[vertex][label=center:$X_2$](x2) at (1,0) {};
				\node[vertex][label=center:$X_3$](x3) at (2.5,0) {};
				\tikzset{EdgeStyle/.style={->}}
				\Edge(A)(B)
				\Edge(A)(C)
				\Edge(A)(x1)
				\Edge(B)(x1)
				\Edge(C)(x1)
				\Edge(A)(x2)
				\Edge(B)(x2)
				\Edge(C)(x2)
				\Edge(A)(x3)
				\Edge(B)(x3)
				\Edge(C)(x3)
				\end{tikzpicture}
			\end{center}
			\caption{DAG of 6 variables where we perform a multiple-vertex intervention.}
			\label{fig:example_hard}
		\end{figure}

\section{On Latent Confounders}
\label{app:latent}
It is well-known that the existence of confounders imposes the most crucial problem for inferring causal relationships from observational data \citep{louizos2017causal, Pearl:2009:CMR:1642718}. 
However, since we perform single-vertex interventions for every node in the CBN, the existence of hidden confounders does not impose a problem.
In the leftmost graph of Figure \ref{fig:confounders}, $X$ and $Y$ are associated observationally due to a hidden common cause, but neither of them is a cause of the other. By intervening $X$ or $Y$, we remove the ``hidden edges''. As a consequence, we are able to infer that neither $X$ nor $Y$ is a cause. The middle graph shows an association between $X$ and $Y$, and the need to intervene $X$ in order to discover that $X$ is a cause of $Y$. Finally, the rightmost graph shows that even in more complex latent configurations, by intervening $X$ we are removing any association between $X$ and $Y$ due to confounders.

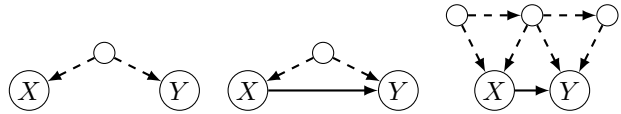
\begin{figure}[ht]
	\begin{center}
		\begin{tikzpicture}[scale=1.0]
		\tikzset{>=latex}
		\tikzstyle{vertex}=[circle, fill=white, draw, inner sep=0pt, minimum size=15pt]
		\node[vertex][label=center:$X$](x0) at (0,0) {};
		\node[vertex][label=center:$Y$](x1) at (2,0) {};
		\node[vertex][minimum size=8pt](x2) at (1,0.5) {};
		\tikzset{EdgeStyle/.style={->,dashed}}
		\Edge(x2)(x0)
		\Edge(x2)(x1)
		\end{tikzpicture}
		\hspace{0.05in}
		\begin{tikzpicture}[scale=1.0]
		\tikzset{>=latex}
		\tikzstyle{vertex}=[circle, fill=white, draw, inner sep=0pt, minimum size=15pt]
		\node[vertex][label=center:$X$](x0) at (0,0) {};
		\node[vertex][label=center:$Y$](x1) at (2,0) {};
		\node[vertex][minimum size=8pt](x2) at (1,0.5) {};
		\tikzset{EdgeStyle/.style={->,dashed}}
		\Edge(x2)(x0)
		\Edge(x2)(x1)
		\tikzset{EdgeStyle/.style={->}}
		\Edge(x0)(x1)
		\end{tikzpicture}
		\hspace{0.05in}
		\begin{tikzpicture}[scale=1.0]
		\tikzset{>=latex}
		\tikzstyle{vertex}=[circle, fill=white, draw, inner sep=0pt, minimum size=15pt]
		\node[vertex][label=center:$X$](x0) at (0.5,0) {};
		\node[vertex][label=center:$Z$](xz) at (1.5,0) {};
		\node[vertex][label=center:$Y$](x1) at (2.5,0) {};
		\node[vertex][minimum size=8pt](x2) at (1,1) {};
		\node[vertex][minimum size=8pt](x3) at (0,1) {};
		\node[vertex][minimum size=8pt](x4) at (2,1) {};
		\tikzset{EdgeStyle/.style={->,dashed}}
		\Edge(x2)(x0)
		\Edge(x2)(x1)
		\Edge(x3)(x0)
		\Edge(x3)(x2)
		\Edge(x4)(xz)
		\Edge(x4)(x1)
		\tikzset{EdgeStyle/.style={->}}
		\Edge(x0)(xz)
		\Edge(xz)(x1)
		\end{tikzpicture}
	\end{center}
	\caption{Examples of a latent configurations that associate the variables $X$ and $Y$.}
	\label{fig:confounders}
\end{figure}

\section{Detailed Proofs}
\label{appendix_proofs}

We now present the proofs of Propositions, Theorems and Corollaries from our main text.

\subsection{Proof of Proposition \ref{lemma:dseparation}}
\label{proof:dseparation}

\begin{proof}
\redcolor{The proof follows directly from rule $3$ of do-calculus \cite{Pearl:2009:CMR:1642718}, which states that $P(\j|do(\i=x_i)) = P(\j)$ if $(\i \perp \j)$ in the mutilated graph after the intervention on $\i$. Since there is no directed path from $i$ to $j$,  in the mutilated graph there is either no path or a path with a v-structure between $i$ and $j$, which implies the independence of $\i$ and $\j$.}

For clarity, we also provide a longer (and equivalent) proof. \redcolor{The proof follows a d-separation argument. Let $\bar{\B}$ be the network after we perform an intervention on $\i$ with value $x_i$, i.e., $\bar{\B}$ has the edge set $E \setminus \{(p_i,i)\ |\ p_i \in \pi_G(i) \}$. Let $anc_G(i)$ and $anc_G(j)$ be the ancestor set of $i$ and $j$ respectively. Now, if there is no directed path from $i$ to $j$ in $\B$ then there is no directed path in $\bar{\B}$ either, therefore, $i \notin anc_G(j)$. Also, $anc_G(i) = \varnothing$ as a consequence of intervening $\i$. Next, we follow the d-separation procedure to determine if $\i$ and $\j$ are marginally independent in $\bar{\B}$. Since $anc_G(i)=\varnothing$, the ancestral graph of $i$ consists of just $i$ itself in isolation, moralizing and disorienting the edges of the ancestral graph of $j$ will not create a path from $i$ to $j$. Thus, guaranteeing the independence of $\i$ and $\j$, i.e., $P(\j) = P(\j|\i)$ in $\bar{\B}$. Finally, since $P(\j|\Uj)$ is fully specified by the parents of $j$ and these parents are not affected by $i$, we have that the marginal of $\j$ in $\B$ remains unchanged in $\bar{\B}$, i.e., $P(\j|do(\i=x_i)) = P(\j)$.}
\end{proof}

\subsection{Proof of Proposition \ref{lemma:probability}}
\label{proof:probability_claims}

\begin{proof}
Here we assume faithfulness in the post-interventional distribution.
Both claims follow a proof by contradiction.
For Claim 1, if for all $x_i \in Dom[\i]$ we have that $P(\j) = P(\j|do(\i=x_i))$ then $\i$ would not be a cause of $\j$, which contradicts the fact that $i \in \pi_\Grm(j)$.
For Claim 2, if for all $x_i,x'_i \in Dom[\i]$ we have that $P(\j|do(\i=x_i)) = P(\j|do(\i=x'_i))$ then in the mutilated graph we have that $P(\j)=P(\j|\i=x_i)$ for all $x_i$, which implies that $\i$ would not be a cause of $\j$, thus contradicting the fact that $i \in \pi_\Grm(j)$.
\end{proof}

\subsection{Proof of Proposition \ref{prop:trees}}
\label{proof:trees}
The proof follows similar arguments to the proof of Theorem \ref{thrm:NRBN_discrete_general}.

\subsection{Proof of Theorem \ref{thrm:NRBN_discrete_general}}
\label{proof:NRBN_discrete_general}

To answer a path query in a discrete CBN, our algorithm compares two empirical PMFs, therefore, we need a good estimation of these PMFs. The following lemma shows the sample complexity to estimate several PMFs simultaneously by using maximum likelihood estimation.
\begin{lemma}
	\label{lemma:tau_condition}
	Let $Y_1,\ldots,Y_{L}$ be $L$ random variables, such that w.l.o.g. the domain of each variable, $Dom[Y_i]$, is a finite subset of $\mathbb{Z}^+$. Also, let $y^{(1)}_i,\ldots,y^{(m)}_i$ be $m$ independent samples of $Y_i$. The maximum likelihood estimator, $\hat{\pbf}(Y_i)$, is obtained as follows:
	\[
	\hat{\prm}_j(Y_i) = \frac{1}{m} \sum_{k=1}^m \indicator[y^{(k)}_i = j], \quad j \in Dom[Y_i].
	\]
	Then, for fixed values of $t > 0$ and $\delta \in (0,1)$, and provided that $m \geq \frac{2}{t^2}\ln \frac{2L}{\delta}$, we have
	\[
	\Pr{ (\forall i \in \{1\ldots L\}) \ \norm{ \hat{\pbf}(Y_i) - \pbf(Y_i) }_\infty \leq t }  \geq 1-\delta.
	\]
\end{lemma}

\hyphenation{Dvo-retz-ky}
\begin{proof}
	We use the Dvoretzky-Kiefer-Wolfowitz inequality \citep*{massart1990tight,dvoretzky1956asymptotic}:
	\[ \Pr{ \sup_{j \in Dom[Y_i]} \abs{ \hat{F}_j(Y_i) - F_j(Y_i) } > t } \leq 2 \exp{-2mt^2}, \quad t > 0,\]
	where \( \hat{F}_j(Y_i) = \sum_{k \leq j} \hat{\prm}_k(Y_i)$ and $F_j(Y_i) = \sum_{k \leq j} \prm_k(Y_i) \). Since \( \hat{\prm}_{j}(Y_i) = \hat{F}_j(Y_i) - \hat{F}_{j-1}(Y_i) \) and \( \prm_{j}(Y_i) = F_j(Y_i) - F_{j-1}(Y_i) \), we have
	\begin{align*}
	\abs{\hat{\prm}_{j}(Y_i) - \prm_{j}(Y_i)} &= \abs{ \left(\hat{F}_j(Y_i) - \hat{F}_{j-1}(Y_i) \right) - \left(F_j(Y_i) - F_{j-1}(Y_i)\right) } \\
	&\leq \abs{ \hat{F}_j(Y_i) - F_j(Y_i) } + \abs{ \hat{F}_{j-1}(Y_i) - F_{j-1}(Y_i) }
	\end{align*}
	therefore, for a specific $i$, we have
	\[ \Pr{ \norm{ \hat{\pbf}(Y_i) - \pbf(Y_i) }_\infty > t } \leq 2 \exp{-mt^2/2}, \quad t > 0. \]
	
	Then by the union bound, we have
	\[ \Pr{ (\exists i \in \{1\ldots L\}) \ \norm{ \hat{\pbf}(Y_i) - \pbf(Y_i) }_\infty > t  } \leq 2L \exp{-mt^2/2}, \quad t > 0. \]
	
	Let \( \delta = 2 L \exp{-mt^2/2} \), then for \( m \geq \frac{2}{t^2}\ln \frac{2L}{\delta} \), we have
	\[ \Pr{ (\forall i \in \{1\ldots L\}) \ \norm{ \hat{\pbf}(Y_i) - \pbf(Y_i) }_\infty \leq t  } \geq 1-\delta, \quad \delta \in (0,1),\ t > 0. \]
	Which concludes the proof of Lemma \ref{lemma:tau_condition}.
\end{proof}

Lemma \ref{lemma:tau_condition} states that simultaneously for all $L$ PMFs, the maximum likelihood estimator $\hat{\pbf}(Y_i)$ is at most $t$-away of $\pbf(Y_i)$ in $\linf$-norm with probability at least $1-\delta$. Next, we provide the proof of Theorem \ref{thrm:NRBN_discrete_general}.

\begin{proof}
We analyze a path query $\tilde{Q}(i,j)$  for nodes $i,j \in \Vrm$. From the contrapositive of Proposition \ref{lemma:dseparation} we have that if $P(\j|do(\i=x_i)) \neq P(\j)$ then there exists a directed path from $i$ to $j$. To detect the latter, we opt to use Claim 2 from Proposition \ref{lemma:probability}.

Let $\pbf_{ij}\si{k} = P(\j|do(\i=x_k))$ for all $i,j \in \Vrm$ and $x_k\in Dom[\i]$, and let  $\hat{\pbf}_{ij}\si{k}$  be the maximum likelihood estimation of $\pbf_{ij}\si{k}$. Also, let $\tau = \frac{\gamma}{2}$ for convenience. Next, using Lemma $\ref{lemma:tau_condition}$ with $t=\tau/4$ and $L=rn^2$, we have
\[
\Pr{ \left( \forall i,j \in \Vrm, \forall x_k \in Dom[\i] \right) \norm{ \hat{\pbf}_{ij}\si{k} - \pbf_{ij}\si{k} }_\infty \leq \tau/4 } \geq 1-\delta.
\]
That is, with probability at least $1-\delta$, simultaneously for all $i,j,k$, the estimators $\hat{\pbf}_{ij}\si{k}$ are at most $\tau/4$-away from the true distributions $\pbf_{ij}\si{k}$ in $\linf$ norm, provided that $m \geq \frac{32}{\tau^2} ( 2 \ln n + \ln \frac{2 r}{\delta})$ samples are used in the estimation.

Now, we analyze the two cases that we are interested to answer with high probability. First, let $i \in \pi_\Grm(j)$. We have that for any two distributions $\pbf_{ij}\si{u},\pbf_{ij}\si{v}$ where $x_u,x_v\in Dom[\i]$, either $\pbf_{ij}\si{u} = \pbf_{ij}\si{v}$ or $ \|\pbf_{ij}\si{u} -\pbf_{ij}\si{v}\|_{\infty} > \tau$ (recall the definition of $\gamma$ and $\tau$). Next, for a specific $i,j$, we show how to test if two distributions $\pbf_{ij}\si{u}, \pbf_{ij}\si{v}$ are equal or not. Let us assume $\pbf_{ij}\si{u} = \pbf_{ij}\si{v}$, then we have
\begin{align*}
\norm{ \hat{\pbf}_{ij}\si{u} -  \hat{\pbf}_{ij}\si{v} }_\infty &= \norm{ \hat{\pbf}_{ij}\si{u} - \pbf_{ij}\si{u} - \left(  \hat{\pbf}_{ij}\si{v} -\pbf_{ij}\si{v} \right) }_\infty \\
&\leq \norm{ \hat{\pbf}_{ij}\si{u} - \pbf_{ij}\si{u} }_\infty + \norm{ \hat{\pbf}_{ij}\si{v} -\pbf_{ij}\si{v}}_\infty \\
&\leq \tau/2.
\end{align*}
Therefore, if $\| \hat{\pbf}_{ij}\si{u} -  \hat{\pbf}_{ij}\si{v}\|_\infty > \tau/2$ then w.h.p. $\pbf_{ij}\si{u} \neq \pbf_{ij}\si{v}$. On the other hand, if $\| \hat{\pbf}_{ij}\si{u} -  \hat{\pbf}_{ij}\si{v}\|_\infty \leq \tau/2$ then w.h.p. we have:
\begin{align*}
\norm{ \pbf_{ij}\si{u} -  \pbf_{ij}\si{v} }_\infty &= \norm{ \pbf_{ij}\si{u} - \hat{\pbf}_{ij}\si{u} - \left(  \pbf_{ij}\si{v} -\hat{\pbf}_{ij}\si{v} \right) + \hat{\pbf}_{ij}\si{u} - \hat{\pbf}_{ij}\si{v} }_\infty \\
&\leq \norm{ \hat{\pbf}_{ij}\si{u} - \pbf_{ij}\si{u}}_\infty + \norm{ \hat{\pbf}_{ij}\si{v} -\pbf_{ij}\si{v}}_\infty + \norm{ \hat{\pbf}_{ij}\si{u} -\hat{\pbf}_{ij}\si{v}}_\infty\\
&\leq \tau.
\end{align*}
From the definition of $\gamma$ and $\tau$, we have $\| \pbf_{ij}\si{u} - \pbf_{ij}\si{v} \|_\infty > \tau$ for any pair $\pbf_{ij}\si{u} \neq \pbf_{ij}\si{v}$, then w.h.p. we have that $\pbf_{ij}\si{u} = \pbf_{ij}\si{v}$.

Second, let be the case that there is no directed path from $i$ to $j$. Then, following Proposition $\ref{lemma:dseparation}$, we have that all the distributions $\pbf_{ij}\si{k}, \forall x_k \in Dom[\i]$, are equal. Similarly as in the first case, we have that if $\| \hat{\pbf}_{ij}\si{u} -  \hat{\pbf}_{ij}\si{v}\|_\infty > \tau/2$ then w.h.p. $\pbf_{ij}\si{u} \neq \pbf_{ij}\si{v}$, and equal otherwise.

Next, note that since Algorithm \ref{algo:query_discrete} compares pair of distributions, the provable guarantee of \textit{all} queries (after eliminating the transitive edges) is directly related to the estimation of \textit{all} PMFs with probability of error at most $\delta$, i.e., we have that
\[
\Pr{ \left(\forall j = 1,\ldots,n \wedge (i \in \pi_\Grm(j) \vee j \notin desc_\Grm(i)) \right)  \tilde{Q}(i,j) = \Q(i,j)} \geq 1-\delta,\]
where $desc_\Grm(i)$ denotes the descendants of $i$. Finally, note that we are estimating each distribution by using $m \geq \frac{32}{\tau^2} ( 2 \ln n + \ln \frac{2 r}{\delta} )$ samples, i.e., $m \in \bO{\frac{1}{\gamma^2} (\ln n + \ln \frac{ r}{\delta})}$. However, for each query $\tilde{Q}(i,j)$ in Algorithm \ref{algo:query_discrete}, we estimate a maximum of $r$ distributions, as a result, we use $\frac{32r}{\tau^2} ( 2 \ln n + \ln \frac{2 r}{\delta} )$ interventional samples in total per query.
\end{proof}

\subsection{Proof of Theorem \ref{thrm:NRBN_continuous_general}}
\label{proof:NRBN_continuous_general}

\begin{proof}
From the contrapositive of Proposition \ref{lemma:dseparation} we have that if $P(\j|do(\i=x_i)) \neq P(\j)$ then there exists a directed path from $i$ to $j$. To detect the latter, we opt to use Claim 1 from Proposition \ref{lemma:probability}, i.e., using expected values. Recall from the characterization of the BN that there exist a finite value $z$ and upper bound $\sigmasub$, such that $\mu(\B,z) \geq 1$ and $\sigma^2(\B,z) \leq \sigmasub$. 
Let $x^{(1)}_j,\ldots, x^{(m)}_j$ be $m$ i.i.d. samples of $\j$ after intervening $\i$ with $z$, and let $\mu_{j|do(\i=z)}$ and $\sigma^2_{j|do(\i=z)}$ be the mean and variance of $\j$ respectively. Also, let $\hmu_{j|do(\i=z)} = \frac{1}{m}\sum_{k=1}^m x^{(k)}_j$ be the empirical expected value of $\j$.

Now, we analyze the two cases that we are interested to answer with high probability. First, let $i \in \pi_\Grm(j)$. Clearly, $\hmu_{j|do(\i=z)}$ has expected value $|\Ev{\hmu_{j|do(\i=z)}}| = |\mu_{j|do(\i=z)}| \geq 1$, and variance $\hat{\sigma}^2_{j|do(\i=z)} = \sigma^2_{j|do(\i=z)}/m \leq \sigmasub/m$. Then, using  Hoeffding's inequality we have
\begin{align}
\Pr{\abs{ \hmu_{j|do(\i=z)} - \mu_{j|do(\i=z)} } \geq t } &\leq  2\exp{-t^2/(2\hat{\sigma}_{j|do(\i=z)}^2)} \nonumber\\
&\leq 2\exp{-mt^2/(2\sigmasub)}. \label{ineq:hoeff_local}
\end{align}

Second, if there is no directed path from $i$ to $j$, then by using Proposition \ref{lemma:dseparation}, we have $\mu_{j|do(\i=z)} = \mu_j = 0$ and $\sigma^2_{j|do(\i=z)} = \sigma^2_j \leq \sigmasub$.

As we can observe from both cases described above, the true mean $\mu_{j|do(\i=z)}$ when $i \in \pi_\Grm(j)$ is at least separated by $1$ from the true mean when there is no directed path. Therefore, to estimate the mean, a suitable value for $t$ in inequality \eqref{ineq:hoeff_local} is $t \leq 1/2$. The latter allows us to state that if $|\hmu_{j|do(\i=z)}| > 1/2$ then $\tilde{Q}(i,j)=1$, and $\tilde{Q}(i,j)=0$ otherwise. Replacing $t=1/2$ and restating inequality \eqref{ineq:hoeff_local}, we have that for a specific pair of nodes $(i,j)$, if $i \in \pi_\Grm(j)$ or if $j \notin desc_\Grm(i)$ ($desc_\Grm(i)$ denotes the descendants of $i$), then	
\[ \Pr{ \Q(i,j) \neq \tilde{Q}(i,j) } \leq  2\exp{-m/(8\sigmasub)}. \]
The latter inequality is for a single query. Using the union bound we have
\[ \Pr{ \left(\exists j = 1,\ldots,n \wedge (i \in \pi_\Grm(j) \vee j \notin desc_\Grm(i)) \right) \   \tilde{Q}(i,j) \neq \Q(i,j) } \leq 2n^2 \exp{-m/(8\sigmasub)}. \]
Now, let $\delta = 2n^2 \exp{-m/(8\sigmasub)}$, if $m \geq 8\sigmasub \log \frac{2n^2}{\delta}$ then
\[ \Pr{ \left(\forall j = 1,\ldots,n \wedge (i \in \pi_\Grm(j) \vee j \notin desc_\Grm(i)) \right)  \tilde{Q}(i,j) = \Q(i,j)} \geq 1-\delta .\]
That is, with probability of at least $1-\delta$, the path query $\tilde{Q}(i,j)$ (in Algorithm \ref{algo:query_continuous}) is equal to $\Q(i,j)$ for all $n^2$ performed queries in which either $i \in \pi_\Grm(j)$, or there is no directed path from $i$ to $j$. Note also that the probability at least $1-\delta$ is guaranteed after we remove the transitive edges in the network. Therefore, we obtain $m \geq 8\sigmasub (2\log n + \log \frac{2}{\delta})$, i.e., $m \in \bO{\sigmasub \log \frac{n}{\delta}}$.
\end{proof}

\subsection{Proof of Theorem \ref{thrm:discrete_transitive}}
The proof follows the same arguments given in the proof of Theorem \ref{thrm:NRBN_discrete_general}. For a pair of nodes $i,j$, Algorithm \ref{algo:transitive_edges} sets $\Srm=\hat{\pi}_{\Grm}(j)$. If $\Srm$ is already the true parent set of $j$, then $\i$ will only have effect on $\j$ if $i \in \Srm$. If $\Srm$ is a subset of the true parent set, then $\i$ will only have effect on $\j$ if there exists a transitive edge $(i,j)$. This is because by intervening $\Srm$ we are blocking any possible effect of $\i$ on $\j$ through any node in $\Srm$, and since non-transitive edges are already recovered then $(i,j)$ must be a transitive edge if there exists some effect. This effect is detected as in Theorem \ref{thrm:NRBN_discrete_general}, i.e., through the $\linf$-norm of difference of  empirical marginals of $\j$.

\subsection{Proof of Theorem \ref{thrm:continuous_transitive}}
The proof follows the same arguments given in the proof of Theorem \ref{thrm:NRBN_continuous_general}. For a pair of nodes $i,j$, Algorithm \ref{algo:transitive_edges} sets $\Srm=\hat{\pi}_{\Grm}(j)$. If $\Srm$ is already the true parent set of $j$, then $\i$ will only have effect on $\j$ if $i \in \Srm$. If $\Srm$ is a subset of the true parent set, then $\i$ will only have effect on $\j$ if there exists a transitive edge $(i,j)$. This is because by intervening $\Srm$ we are blocking any possible effect of $\i$ on $\j$ through any node in $\Srm$, and since non-transitive edges are already recovered then $(i,j)$ must be a transitive edge if there exists some effect. This effect is detected as in Theorem \ref{thrm:NRBN_continuous_general}, i.e., through the absolute value of the difference of the empirical means of $\j$.

\subsection{Proof of Theorem \ref{thrm:imperfect_discrete}}
\label{proof:imperfect_discrete}

To prove Theorem \ref{thrm:imperfect_discrete} we first derive a lemma that specifies the number of samples to obtain a good approximation with guarantees of conditional PMFs.

\begin{lemma}
\label{lemma:noisy_mle}
Let $Y_1,\ldots,Y_{L}$ be $L$ discrete random variables, such that w.l.o.g. the domain of each variable, $Dom[Y_i]$, is a finite subset of $\mathbb{Z}^+$. Let $Z_1,\ldots,Z_{L}$ be $L$ Bernoulli random variables, such that each variable fulfills $P(Z_i = 1) \geq \alpha \geq 1/2$. Also, let $(z^{(1)}_i, y^{(1)}_i),\ldots,(z^{(m)}_i, y^{(m)}_i)$ be $m$ pair of independent samples of $Z_i$ and $Y_i$. The conditional maximum likelihood estimator, $\hat{\pbf}(Y_i|Z_i=1)$, is obtained as follows:
\[
\hat{\prm}_j(Y_i|Z_i=1) = \frac{1}{\sum_{k=1}^m z^{(k)}_i} \sum_{k=1}^m \indicator[y^{(k)}_i = j \wedge z^{(k)}_i ], \quad j \in Dom[Y_i].
\]
Then, for fixed values of $t, \delta \in (0,1)$, and provided that $m \geq \frac{4}{\alpha t^2}\ln \frac{4L}{\delta}$, we have
\[
\Pr{ (\forall i \in \{1\ldots L\}) \ \norm{ \hat{\pbf}(Y_i|Z_i=1) - \pbf(Y_i|Z_i=1) }_\infty \leq t  } \geq 1-\delta.
\]
\end{lemma}

\begin{proof}
First, we analyze a pair of variables $Z_i, Y_i$. Let $\Event_1 = \{ \frac{1}{m}\sum_{k=1}^m z_i^{(k)} \geq \alpha - \epsilon \}$. Next, using the one-sided Hoeffding's inequality, we have
\[
P(\Event_1) \geq 1 - e^{-2\epsilon^2 m}.
\]
Now, let the event $  \Event_2 = \{ \lVert \hat{\pbf}(Y_i|Z_i=1) - \pbf(Y_i|Z_i=1) \rVert_\infty \leq t \}$. Using Lemma \ref{lemma:tau_condition} (see Proof \ref{proof:NRBN_discrete_general}), we obtain
\[
P(\Event_2 | \Event_1) \geq 1 - 2 e^{-m (\alpha-\epsilon) t^2 / 2}.
\]
Then, by the law of total probability, we have
\begin{align*}
P(\Event_2) &\geq P(\Event_2|\Event_1)P(\Event_1) \\
            &\geq 1 - e^{-2\epsilon^2 m} - 2 e^{-m (\alpha-\epsilon) t^2 / 2}.
\end{align*}
Let $\frac{\delta}{2} = e^{-2\epsilon^2 m}$, and $\frac{\delta}{2} = 2 e^{-m (\alpha-\epsilon) t^2 / 2}$. Then provided that $m \geq \max(\frac{1}{2\epsilon^2} \ln \frac{2}{\delta}, \frac{2}{(\alpha - \epsilon)t^2} \ln \frac{4}{\delta})$,
\[
P(\Event_2) \geq 1 - \delta.
\]
For $\epsilon = \frac{\alpha}{2}$, and $t \in (0,1)$, we can simplify the bound on $m$ to be $m \geq \frac{4}{\alpha t^2} \ln \frac{4}{\delta}$. Finally, using union bound and provided that $m \geq \frac{4}{\alpha t^2} \ln \frac{4L}{\delta}$, we have
\[
\Pr{ (\forall i \in \{1\ldots L\}) \ \norm{ \hat{\pbf}(Y_i|Z_i=1) - \pbf(Y_i|Z_i=1) }_\infty \leq t  } \geq 1-\delta.
\]
Which concludes the proof.
\end{proof}

Now follows the proof of  Theorem \ref{thrm:imperfect_discrete}.

\begin{proof}[Proof of Theorem \ref{thrm:imperfect_discrete}]
The proof follows the same steps as in the proof of Theorem \ref{thrm:NRBN_discrete_general} (Appendix \ref{proof:NRBN_discrete_general}). The difference is that we now use the sample complexity given by Lemma \ref{lemma:noisy_mle} instead of Lemma \ref{lemma:tau_condition}. Therefore,  for a query $\tilde{Q}(i,j)$ we obtain a sample complexity of $m \in \bO{\frac{1}{\alpha\gamma^2} \left(  \ln n + \ln \frac{ r}{\delta} \right)}$.
\end{proof}

\subsection{Proof of Theorem \ref{thrm:imperfect_continuous}}
\label{proof:imperfect_continuous}

\begin{proof}
Recall from the characterization of the BN that there exist a finite value $z$ and upper bound $\sigmasub$, such that $\mu(\B,z) \geq 1$ and $\sigma^2(\B,z) \leq \sigmasub$. 
Let $x^{(1)}_j,\ldots, x^{(m)}_j$ be $m$ i.i.d. samples of $\j$ after trying to intervene $\i$ with value $z$. Let $\mu_{j|do(\i=z)}$ and $\sigma^2_{j|do(\i=z)}$ be the mean and variance of $\j$ respectively, after perfectly intervening $\i$ with value $z$. Also, let $\hmu = \frac{1}{m}\sum_{k=1}^m x^{(k)}_j$ be the empirical expected value of $\j$.

Now, we analyze the two cases that we are interested to answer with high probability. First, let $i \in \pi_\Grm(j)$. Clearly, $\hmu$ has expected value $|\Ev{\hmu}| = |\mathbb{E}_{\i}[ \mu_{j|do(\i=z)} ]| \geq 1$, and variance $\hat{\sigma}^2 = \mathbb{E}_{\i}[ \sigma^2_{j|do(\i=z)} ]/m \leq \sigmasub/m$. Then, using  Hoeffding's inequality we have
\begin{align}
\Pr{\abs{ \hmu - \Ev{\hmu} } \geq t } &\leq  2\exp{-t^2/(2\hat{\sigma}^2)} \nonumber\\
&\leq 2\exp{-mt^2/(2\sigmasub)}. \label{ineq:hoeff_imperfect}
\end{align}

Second, if there is no directed path from $i$ to $j$, then by using Proposition \ref{lemma:dseparation}, we have $\mathbb{E}_{\i}[\mu_{j|do(\i=z)}] = \mathbb{E}_{\i}[\mu_j] = 0$ and $\mathbb{E}_{\i}[\sigma^2_{j|do(\i=z)}] = \mathbb{E}_{\i}[\sigma^2_j] \leq \sigmasub$.

As we can observe from both cases described above, the true mean $\mathbb{E}_{\i}[\mu_{j|do(\i=z)}]$ when $i \in \pi_\Grm(j)$ is at least separated by $1$ from the true mean when there is no directed path. Therefore, to estimate the mean, a suitable value for $t$ in inequality \eqref{ineq:hoeff_imperfect} is $t \leq 1/2$. The latter allows us to state that if $|\hmu| > 1/2$ then $\tilde{Q}(i,j)=1$, and $\tilde{Q}(i,j)=0$ otherwise. Replacing $t=1/2$ and restating inequality \eqref{ineq:hoeff_imperfect}, we have that for a specific pair of nodes $(i,j)$, if $i \in \pi_\Grm(j)$ or if $j \notin desc_\Grm(i)$ ($desc_\Grm(i)$ denotes the descendants of $i$), then	
\[ 
\Pr{ \Q(i,j) \neq \tilde{Q}(i,j) } \leq  2\exp{-m/(8\sigmasub)}. 
\]
The latter inequality is for a single query. Using the union bound we have
\[
 \Pr{ \left(\exists j = 1,\ldots,n \wedge (i \in \pi_\Grm(j) \vee j \notin desc_\Grm(i)) \right) \   \tilde{Q}(i,j) \neq \Q(i,j) } \leq 2n^2 \exp{-m/(8\sigmasub)}. 
 \]
Now, let $\delta = 2n^2 \exp{-m/(8\sigmasub)}$, if $m \geq 8\sigmasub \log \frac{2n^2}{\delta}$ then
\[ \Pr{ \left(\forall j = 1,\ldots,n \wedge (i \in \pi_\Grm(j) \vee j \notin desc_\Grm(i)) \right)  \tilde{Q}(i,j) = \Q(i,j)} \geq 1-\delta .\]
That is, with probability of at least $1-\delta$, the path query $\tilde{Q}(i,j)$ (in Algorithm \ref{algo:query_continuous}) is equal to $\Q(i,j)$ for all $n^2$ performed queries in which either $i \in \pi_\Grm(j)$, or there is no directed path from $i$ to $j$. Note also that the probability at least $1-\delta$ is guaranteed after we remove the transitive edges in the network. Therefore, we obtain $m \geq 8\sigmasub (2\log n + \log \frac{2}{\delta})$, i.e., $m \in \bO{\sigmasub \log \frac{n}{\delta}}$.
\end{proof}

\subsection{Proof of Corollary \ref{corol:NRBN_subgaussian}}
\label{proof:NRBN_subgaussian}

\begin{proof}
Let us first analyze the expected value $\mu_j$ of each variable $\j$ in the network before performing any intervention. From the definition of the ASGN model we have that the expected value of $\j$ is $\mu_j = \sum_{p \in \pi_\Grm(j)} \Wrm_{jp} \mu_p$, and from the topological ordering of the network we can observe that the variables without parents have zero mean since these are only affected by a sub-Gaussian noise with zero mean. Therefore, following this ordering we have that the mean of every variable $\j$ is $\mu_j = 0$.

Recall from Remark \ref{remark:asgm} that we can write the model as: $X = \W X + N$, which is equivalent to $X = (\I - \W)^{-1}N$. Let $\BB = (\I - \W)^{-1}$, then $\BB_{ji}$ denotes the total \emph{weight effect} of  the noise $N_i$ on the node $j$. Furthermore, let  $\odot_i \BB = (\I - \odot_i\W)^{-1}$ and similarly $\{\odot_i \BB\}_{jk}$ denotes the total weight effect of  the noise $N_k$ on the node $j$ after intervening the node $i$.  

Next, we analyze if $z = 1/\wmin$, and $\sigmasub = \noisemax w_{max}$ fulfill the conditions given in Theorem \ref{thrm:NRBN_continuous_general}.
First, let $i \in \pi_\Grm(j)$, i.e., $(i,j)\in \Erm$.  Since $\wmin = \min_{(i,j)\in \Erm} |\{\odot_i \BB\}_{ji}|$, we have $| \mu_{j|do(\i=z)} | = |\{\odot_i \BB\}_{ji}| \times |z| = |\{\odot_i \BB\}_{ji}|/\wmin$. Since $\wmin \leq |\{\odot_i \BB\}_{ji}|$ for any $(i,j) \in \Erm$, we have that $\mu(\B,z) \geq 1$. Let $\upsilon_{j|do(\i=z)}$ be the variance of $\j$ after intervening $\i$, then we have that
$ \upsilon^2_{j|do(\i=z)} =  \sum_{p \in \Vrm \setminus {i}} (\{ \odot_i \BB\ \}_{jp})^2  \noisej,
$
similarly, the variance of $j$ without any intervention is 
$\upsilon^2_j = \sum_{p \in \Vrm \setminus {i}} ( \BB_{jp} )^2  \noisej.
$
Then $\max_{(i,j)\in \Erm} \upsilon^2_{j|do(\i=z)} \leq \max_{i \in \Vrm} \noisemax \lVert \odot_i \BB \rVert^2_{\infty,2}$, and $\max_{j \in \Vrm} \upsilon^2_j \leq  \noisemax \lVert  \BB \rVert^2_{\infty,2}$, which results in $\sigmasub = \noisemax w_{max}$.

Second, let be the case that there is no directed path from $i$ to $j$. Then from Proposition \ref{lemma:dseparation}, $\i$ and $\j$ are independent after intervening $\i$, i.e., $\mu_{j|do(\i=z)} = \mu_j = 0$, and $\upsilon^2_{j|do(\i=z)} = \upsilon^2_j \leq \sigmasub$.

As shown above, for these values of $z= 1/\wmin$ and $\sigmasub = \noisemax \wmax$, we fulfill the conditions given in Theorem \ref{thrm:NRBN_continuous_general}, which concludes our proof.
\end{proof}

\subsection{Proof of Corollary \ref{corol:NRBN_subgaussian_transitive}}
\label{proof:corollary_transitive}

For a pair of nodes $i,j$, Algorithm \ref{algo:transitive_edges} sets $\Srm=\hat{\pi}_{\Grm}(j)$. If $\Srm$ is already the true parent set of $j$, then $\i$ will only have effect on $\j$ if $i \in \Srm$. If $\Srm$ is a subset of the true parent set, then $\i$ will only have effect on $\j$ if there exists a transitive edge $(i,j)$. This is because by intervening $\Srm$ we are blocking any possible effect of $\i$ on $\j$ through any node in $\Srm$, and since non-transitive edges are already recovered then $(i,j)$ must be a transitive edge if there exists some effect. Thus,  $\wmin = \min_{ij} |\Wrm_{ij}|$ is enough to ensure a mean of at least $1$ for $\j$, since only $\i$ is intervened with value $z_2=1/\wmin$ while the other nodes in $\Srm$ are intervened with value $z_1=0$. Finally, because the value of $\wmax$ takes the maximum across all possible interventions of subsets of the parent set of $j$, then $\sigmasub$ is an upper bound and similar arguments as in Corollary \ref{corol:NRBN_subgaussian} hold.

\subsection{Proof of Corollary \ref{corollary:imperfect_NRBN_subgaussian}}
\label{proof:imperfect_NRBN_subgaussian}

\begin{proof}
To prove the corollary we need to show that for $z=1/\wmin$ and $\sigmasub = \noisemax \wmax$, the conditions $\mu(\B,z) \geq 1$ and $\sigma^2(\B,z) \leq \sigmasub$ hold, similarly to Proof \ref{proof:NRBN_subgaussian}.

For the case when $i \in \pi_\Grm(j)$, now $\i$ (the intervened variable)  is a sub-Gaussian variable with mean $z$ and variance $\nu_i^2$, we clearly have that the same upper bound $\sigmaub= \noisemax \wmax$ works since $\nu^2_i \leq \noisemax$. Likewise, the value $z$ is properly set since the value of $\wmin$ is  $\wmin = \min_{(i,j)\in \Erm} |\{\odot_i \BB\}_{ji}|$.

For the case when there is no directed path from $i$ to $j$, we have that $\i$ and $\j$ are independent after intervening $\i$, i.e., $\Ev{\j} = \mu_j = 0$, and $\Var{\j} = \upsilon^2_j \leq \sigmasub$.

From these analyses we conclude that the ASGN model fulfills the conditions given in Theorem \ref{thrm:imperfect_continuous}. Which concludes our proof.
\end{proof}

\section{Experiments}
\label{app:experiments_section}

\subsection{Experiments on Synthetic CBNs}
\label{appendix:experiments_synthetic}

In this section, we validate our theoretical results on synthetic data for perfect and imperfect interventions by using Algorithms \ref{algo:learning_tbn}, \ref{algo:query_discrete}, and \ref{algo:query_continuous}. Our objective is to characterize the number of interventional samples per query needed by our algorithm for learning the transitive reduction of a CBN exactly. 

Our experimental setup is as follows. We sample a random transitively reduced DAG structure $\Grm$ over $n$ nodes.
We then generate a CBN as follows: for a discrete CBN, the domain of a variable $\i$ is $Dom[\i] = \{1,\ldots,d\}$, where $d$ is the size of the domain, which is selected uniformly at random from $\{2,\ldots,5\}$, i.e., $r=5$ in terms of Theorem \ref{thrm:NRBN_discrete_general}. Then, each row of a CPT is generated uniformly at random. Finally, we ensure that the generated CBN fulfills $\gamma \geq 0.01$.
For a continuous CBN, we use Gaussian noises following the ASGN model as described in Definition \ref{def:subgaussian}, where each noise variable $N_i$ is Gaussian with mean $0$ and variance selected uniformly at random from $[1,5]$, i.e., $\noisemax = 5$, in terms of Corollary \ref{corol:NRBN_subgaussian}. The edge weights $\Wrm_{ij}$ are selected uniformly at random from $[-1.25,-0.01] \cup [0.01,1.25]$ for all $(i,j) \in \Erm$. We ensure that $\W$ fulfills $\lVert(\I-\W)^{-1}\rVert^2_{2,\infty} \leq 20$. 
After generating a CBN, one can now intervene a variable, and sample accordingly to a given query. Finally, we set $\delta = 0.01$, and estimate the probability $P(\Grm=\hat{\Grm})$ by computing the fraction of times that the learned DAG structure $\hat{\Grm}$ matched the true DAG structure $\Grm$ exactly, across 40 randomly sampled BNs. We repeated this process for $n \in \{20,40,60\}$. 
The number of samples per query was set to $e^C\log nr$ for discrete BNs, and $e^C\log n$ for continuous BNs, where $C$ was the control parameter, chosen to be in $[0,16]$. Figure \ref{fig:experiments_appendix} shows the results of the structure learning experiments. We can observe that there is a sharp phase transition from recovery failure to success in all cases, and that the $\log n$ scaling holds in practice, as prescribed by Theorems \ref{thrm:NRBN_discrete_general} and \ref{thrm:NRBN_continuous_general}.

Similarly, for imperfect interventions we work under the same experimental settings described above.
For a discrete BN, we additionally set $\alpha = 0.9$ in terms of Theorem \ref{thrm:imperfect_discrete}. 
Whereas for a continuous BN, we  set $\nu^2_i = \sigma^2_i$ for all $i \in \Vrm$, in terms of \ref{corollary:imperfect_NRBN_subgaussian}. Figure \ref{fig:experiments_appendix} shows the results of the structure learning experiments. We can observe that the sharp phase transition from recovery failure to success and the $\log n$ scaling is also preserved, as prescribed by Theorems \ref{thrm:imperfect_discrete} and  \ref{thrm:imperfect_continuous}.

\begin{figure}[!ht]
	\centering
	\begin{minipage}[c]{0.3\linewidth}
		\centering
		\includegraphics[width=\linewidth]{./figures/new_discrete.pdf}
	\end{minipage}
	\hspace{0.5in}
	\begin{minipage}[c]{0.3\linewidth}
		\centering
		\includegraphics[width=\linewidth]{./figures/new_continuous.pdf}  
	\end{minipage}
	
	\begin{minipage}[c]{0.3\linewidth}
		\centering
		\includegraphics[width=\linewidth]{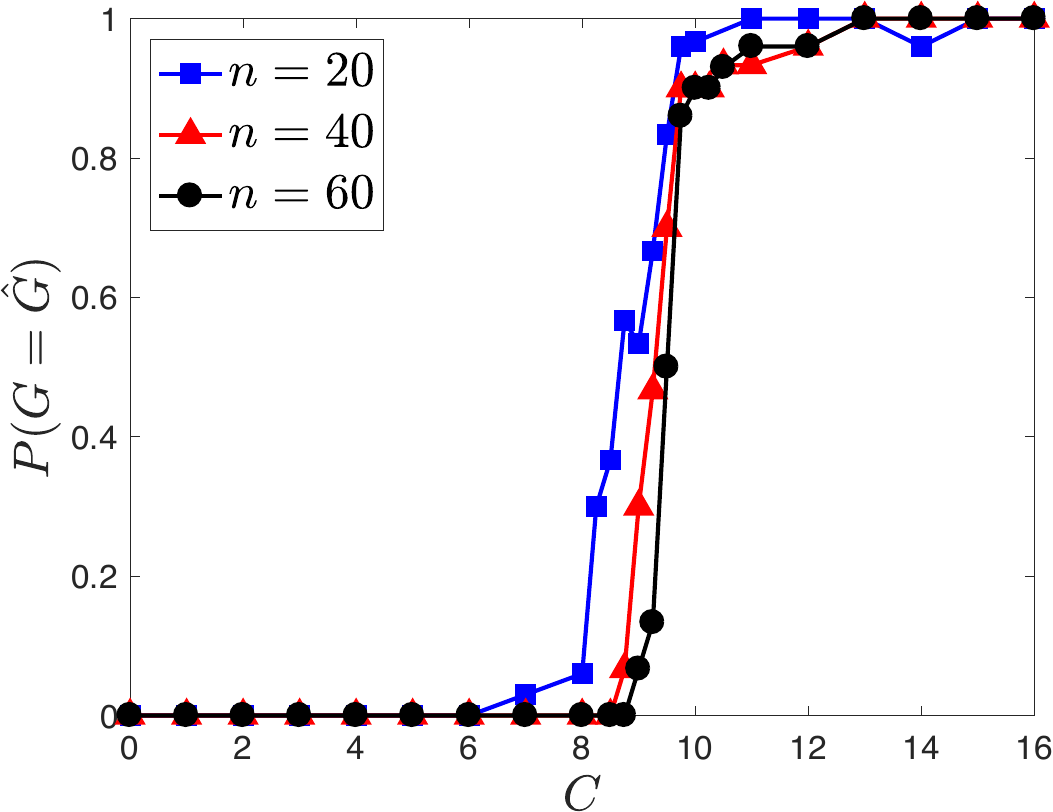}
	\end{minipage}
	\hspace{0.5in}
	\begin{minipage}[c]{0.3\linewidth}
		\centering
		\includegraphics[width=\linewidth]{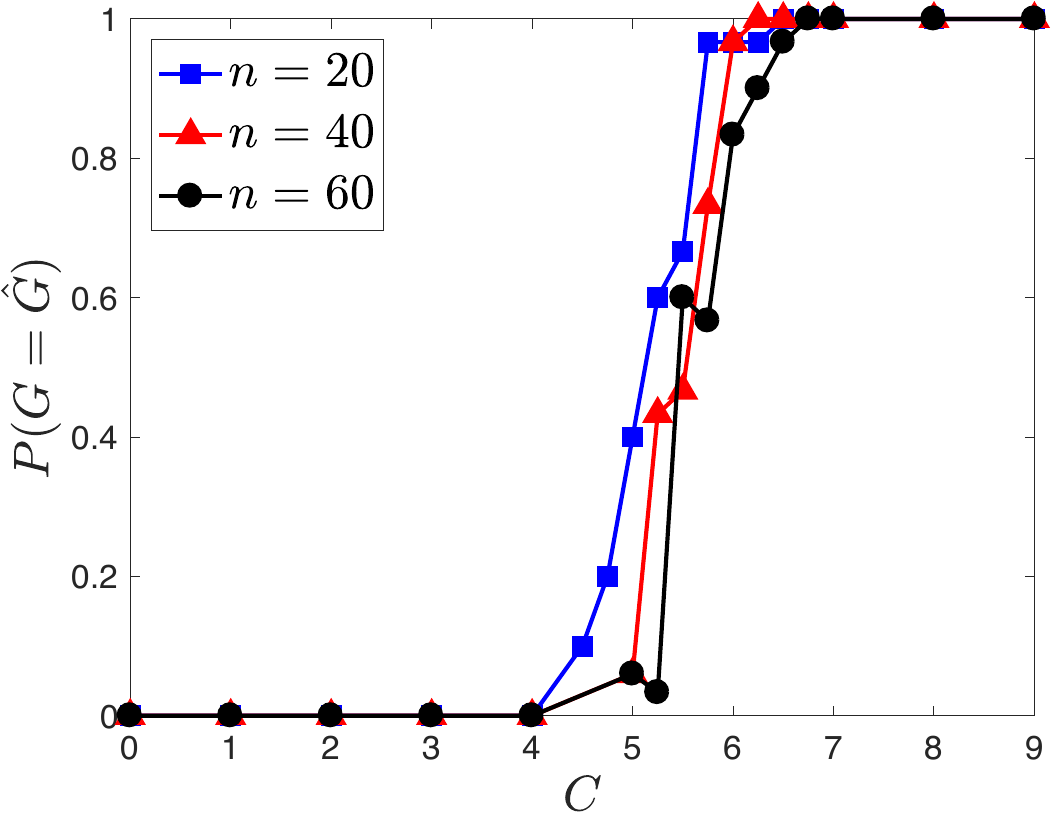}  
	\end{minipage}

	\caption{(Left, Top) Probability of correct structure recovery of the transitive reduction of a discrete CBN vs. number of samples per query, where the latter was set to $e^C\log nr$, with all CBNs having $r=5$ and $\gamma \geq 0.01$. (Right, Top) Similarly, for continuous CBNs, the number of samples per query was set to $e^C\log n$, with all CBNs having $\lVert(\I-\W)^{-1}\rVert^2_{2,\infty} \leq 20$. (Left, Bottom) Results for imperfect interventions for discrete CBNs under same settings as in perfect interventions and $\alpha =0.9$. (Right, Bottom) Results for imperfect interventions for continuous CBNs under same settings as in perfect interventions and $\nu^2_i = \sigma^2_i, \forall i \in V$. Finally, we observe that there is a sharp phase transition from recovery failure to success in all cases, and the $\log n$ scaling holds in practice, as prescribed by Theorems \ref{thrm:NRBN_discrete_general}, \ref{thrm:NRBN_continuous_general}, \ref{thrm:imperfect_discrete}, and \ref{thrm:imperfect_continuous}.}
	\label{fig:experiments_appendix}
\end{figure}

\subsection{Most Benchmark BNs Have Few Transitive Edges}
\label{appendix_proportion}

In this section we compute some attributes of $21$   benchmark networks, which are publicly available at \url{http://compbio.cs.huji.ac.il/Repository/networks.html} and \url{http://www.bnlearn.com/bnrepository/}.  These  benchmark BNs contain the DAG structure and the conditional probability tables.  Several prior works also used these BNs and  evaluated DAG recovery by sampling data \emph{observationally} by using the joint probability distribution \citep{Brenner13,Tsamardinos06}.

Table \ref{tab:proportions} reports the number of vertices, $|\Vrm|$, the number of edges, $|\Erm|$, the number of transitive edges, $|\mathrm{RE}|$, and the ratio, $\mathrm{|RE|}/|\Erm|$. Finally, the mean and median of the ratios is presented. A median of $0.48\%$ indicates that more than half of these networks have a number of transitive edges less than $0.50\%$ of the total number of edges. In other words, our methods provide guarantees for exact learning of at least $99.5\%$ of the true structure for many of these benchmark networks.

\begin{table}[!ht]
\centering
\caption{For each network we show the number of vertices, $|\Vrm|$, the number of edges, $|\Erm|$, the number of transitive edges, $|\mathrm{RE}|$, and the ratio, $|\mathrm{RE}|/|\Erm|$.}
\label{tab:proportions}
\begin{tabular}{@{}lrrrr@{}}
\toprule
Network    & $|\Vrm|$ & $|\Erm|$ & $|\mathrm{RE}|$ & $|\mathrm{RE}|/|\Erm|$ \\ \midrule
Alarm & 37 & 46 & 4 & 8.70\% \\
Andes & 223 & 338 & 45 & 13.31\% \\
Asia & 8 & 8 & 0 & 0.00\% \\
Barley & 48 & 84 & 14 & 16.67\% \\
Cancer & 5 & 4 & 0 & 0.00\% \\
Carpo & 60 & 74 & 0 & 0.00\% \\
Child & 20 & 25 & 1 & 4.00\% \\
Diabetes & 413 & 602 & 48 & 7.97\% \\
Earthquake & 5 & 4 & 0 & 0.00\% \\
Hailfinder & 56 & 66 & 4 & 6.06\% \\
Hepar2 & 70 & 123 & 16 & 13.01\% \\
Insurance & 27 & 52 & 12 & 23.08\% \\
Link & 724 & 1125 & 0 & 0.00\% \\
Mildew & 35 & 46 & 6 & 13.04\% \\
Munin1 & 186 & 273 & 1 & 0.37\% \\
Munin2 & 1003 & 1244 & 6 & 0.48\% \\
Munin3 & 1041 & 1306 & 6 & 0.46\% \\
Munin4 & 1038 & 1388 & 6 & 0.43\% \\
Pigs & 441 & 592 & 0 & 0.00\% \\
Water & 32 & 66 & 0 & 0.00\% \\
Win95pts & 76 & 112 & 8 & 7.14\% \\
\hline
Average & & & & 5.46\%\\
Median & & & & 0.48\%\\
\bottomrule
\end{tabular}
\end{table}

\subsection{DAG Recovery on Benchmark BNs}
\label{appendix_experiments}

In this section we test Algorithms \ref{algo:learning_tbn}, \ref{algo:query_discrete}, \ref{algo:query_continuous}, \ref{algo:query_discrete_transitive}, \ref{algo:query_continuous_transitive}, and \ref{algo:transitive_edges}, on  benchmark networks that may contain transitive edges. The networks are publicly available at \url{http://www.bnlearn.com/bnrepository/}. These standard  benchmark BNs contain the DAG structure and the conditional probability distributions. We sample data \emph{interventionally} by using the manipulation theorem \citep{Pearl:2009:CMR:1642718}. We then compare the learned DAG versus the true DAG. Several prior works used these BNs and also evaluated DAG recovery by sampling data \emph{observationally} by using the joint probability distribution \citep{Brenner13,Tsamardinos06}. 

\paragraph{Discrete networks.}
We first present experiments on discrete BNs. For each network we set the number of samples $m = e^{12} \log nr$, and ran Algorithm \ref{algo:learning_tbn} once. After learning the transitive reduction, we ran Algorithm \ref{algo:transitive_edges} to learn the missing transitive edges.
For the true edge set $\Erm$ and recovered edge set $\tilde{\Erm}$, we define the edge precision as $|\tilde{\Erm} \cap \Erm|/|\tilde{\Erm}|$, and the edge recall as $|\tilde{\Erm} \cap \Erm|/|\Erm|$. The F1 score was computed from the previously defined precision and recall. 
As we can observe in Table \ref{tab:realworld}, all of the networks achieved an edge precision of $1.0$, which indicates that all the edges that our algorithm learned are indeed part of the true network. Finally, all networks also achieved an  edge recall of $1.0$, which indicates that all edges (including the transitive edges) were correctly recovered. 

\begin{table}[!ht]
\centering
\caption{Results on benchmark discrete networks. For each network, we show the number of nodes, $n$, the number of edges, $|\Erm|$, the number of transitive edges, $|\mathrm{RE}|$, the maximum domain size, $r$, the edge precision, $|\tilde{\Erm} \cap \Erm|/|\tilde{\Erm}|$, the edge recall, $|\tilde{\Erm} \cap \Erm|/|\Erm|$, and the F1 score.}
\label{tab:realworld}
\begin{tabular}{@{}lrrrrccc@{}}
\toprule
Network    & $n$ & $|\Erm|$ & $|\mathrm{RE}|$ & $r$ & \begin{tabular}[c]{@{}c@{}}Edge\\ precision\end{tabular} & \begin{tabular}[c]{@{}c@{}}Edge\\ recall\end{tabular} & F1 score \\ \midrule
Carpo      & $60$  & $74$    & $0$                                                                    & $4$   & $1.00$                                                    & $1.00$                                                 & $1.00$    \\
Child      & $20$  & $25$    & $1$                                                                    & $6$   & $1.00$                                                    & $1.00$                                                 & $1.00$    \\
Hailfinder & $56$  & $66$    & $4$                                                                    & $11$  & $1.00$                                                    & $1.00$                                                 & $1.00$    \\
Win95pts   & $76$  & $112$   & $8$                                                                    & $2$   & $1.00$                                                    & $1.00$                                                 & $1.00$    \\ \bottomrule
\end{tabular}
\end{table}

\paragraph{Additive Gaussian networks.}
Next, we present experiments on continuous BNs. For each network we set the number of samples $m = e^{C} \log n$, and ran Algorithm \ref{algo:learning_tbn} once. 
For the true edge set $\Erm$ and recovered edge set $\tilde{\Erm}$, we define the edge precision as $|\tilde{\Erm} \cap \Erm|/|\tilde{\Erm}|$, and the edge recall as $|\tilde{\Erm} \cap \Erm|/|\Erm|$. The F1 score was computed from the previously defined precision and recall. 
As we can observe in Table \ref{tab:realworld_gaussian},  both networks achieved an edge precision of $1.0$, which indicates that all the edges that our algorithm learned are indeed part of the true network. Finally, both networks also achieved an  edge recall of $1.0$, which indicates that all edges (including the transitive edges) were correctly recovered. 

\begin{table}[!ht]
	\centering
	\caption{Results on benchmark continuous networks. For each network, we show the number of nodes, $n$, the number of edges, $|\Erm|$, the number of transitive edges, $|\mathrm{RE}|$, the constant $C$, the maximum domain size, $r$, the edge precision, $|\tilde{\Erm} \cap \Erm|/|\tilde{\Erm}|$, the edge recall, $|\tilde{\Erm} \cap \Erm|/|\Erm|$, and the F1 score.}
	\label{tab:realworld_gaussian}
	\begin{tabular}{@{}lrrrrccc@{}}
		\toprule
		Network    & $n$ & $|\Erm|$ & $|\mathrm{RE}|$ & $C$ & \begin{tabular}[c]{@{}c@{}}Edge\\ precision\end{tabular} & \begin{tabular}[c]{@{}c@{}}Edge\\ recall\end{tabular} & F1 score \\ \midrule
		Magic-Irri      & $64$  & $102$    & $25$ & $11$ & $1.00$ & $1.00$ & $1.00$    \\
		Magic-Niab      & $44$  & $66$    & $12$  & $7$ & $1.00$ & $1.00$ & $1.00$    \\
		\bottomrule
	\end{tabular}
\end{table}

\subsection{DAG Recovery on Real-World Gene Perturbation Datasets}
\label{appendix_nature}

In this section we show experimental results on real-world interventional data. We selected 14 yeast genes from the gene perturbation data in ``Transcriptional regulatory code of a eukaryotic genome'' \citep{harbison2004transcriptional}. A few observations from the learned BN shown in Figure \ref{fig:gene} are: the gene YFL044C reaches 2 genes directly and has an indirect influence on all 11 remaining genes; finally, the genes YML081W and YNR063W are reached by almost all other genes.

\begin{figure}[!ht]
	\centering
    \includegraphics[width=0.3\linewidth]{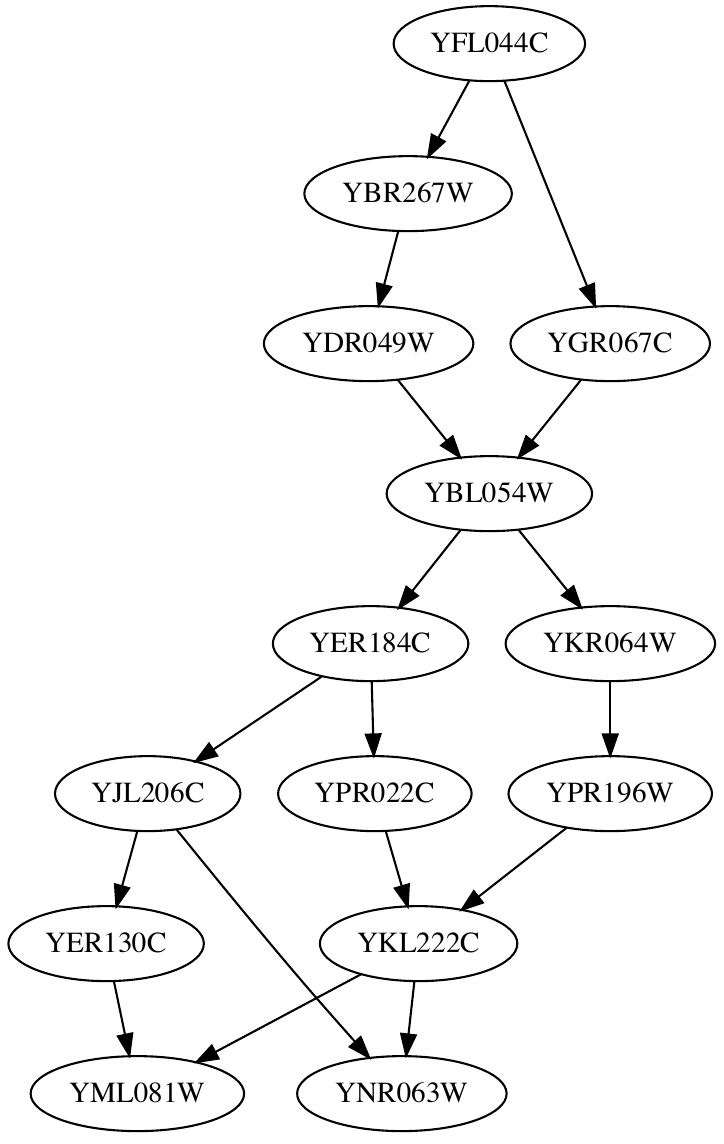}  
	\caption{DAG structure recovered from interventional data in \cite{harbison2004transcriptional}. The nodes correspond to yeast genes.}
	\label{fig:gene}
\end{figure}

Next we show experimental results on real-world gene perturbation data from \citet{xiao2015gene}. 
Figure \ref{fig:mouse} shows the learned DAGs for genes from mouses (Left) and humans (Right).
For mouse genes we analyzed 17 genes and we can observe the following: the gene Spint1 reaches 3 genes directly and all other genes indirectly; finally, the genes Tgm2, Ifnb1, Tgfbr2 and Hmgn1 are the most influenced genes.
For human genes we analyzed 17 genes and we observe the following:  the gene CTGF reaches 1 gene directly and all the remaining genes indirectly; finally, the gene HNRNPA2B1 is reached by all genes.

\begin{figure}[!ht]
	\centering
	
	\begin{minipage}[c]{0.45\linewidth}
		\centering
		\includegraphics[width=0.65\linewidth]{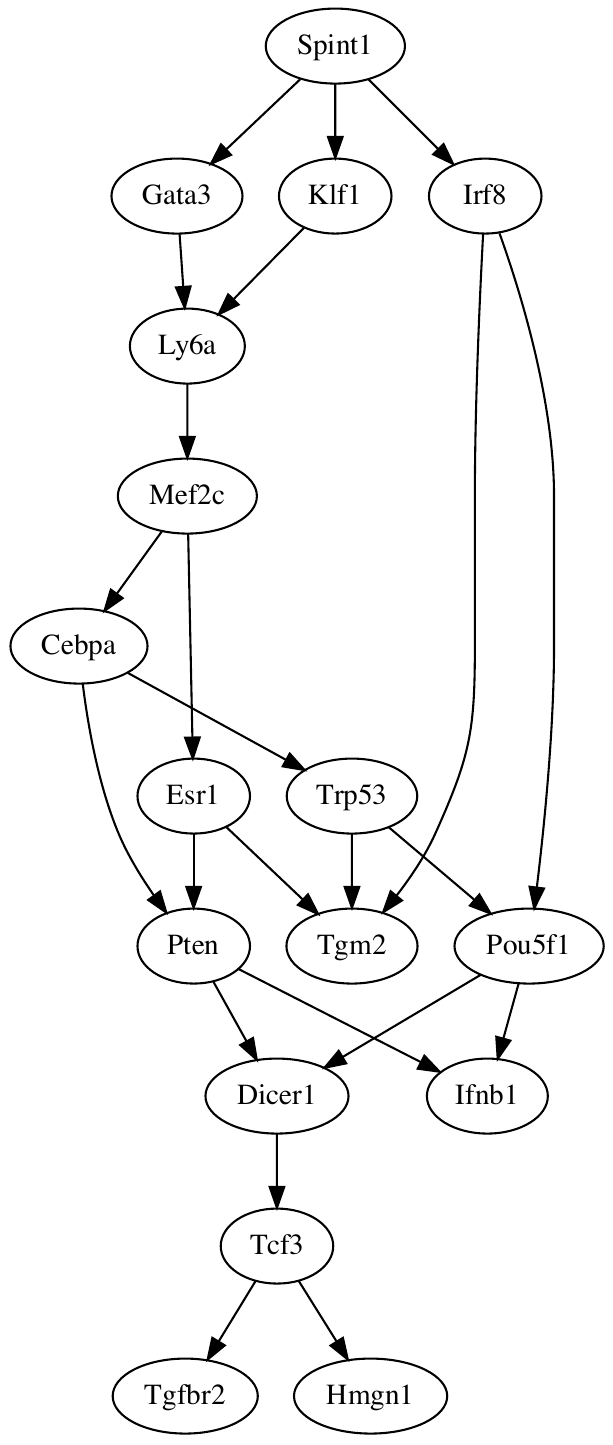}
	\end{minipage}
	\hspace{0.1in}
	\begin{minipage}[c]{0.45\linewidth}
		\centering
		\includegraphics[width=0.42\linewidth]{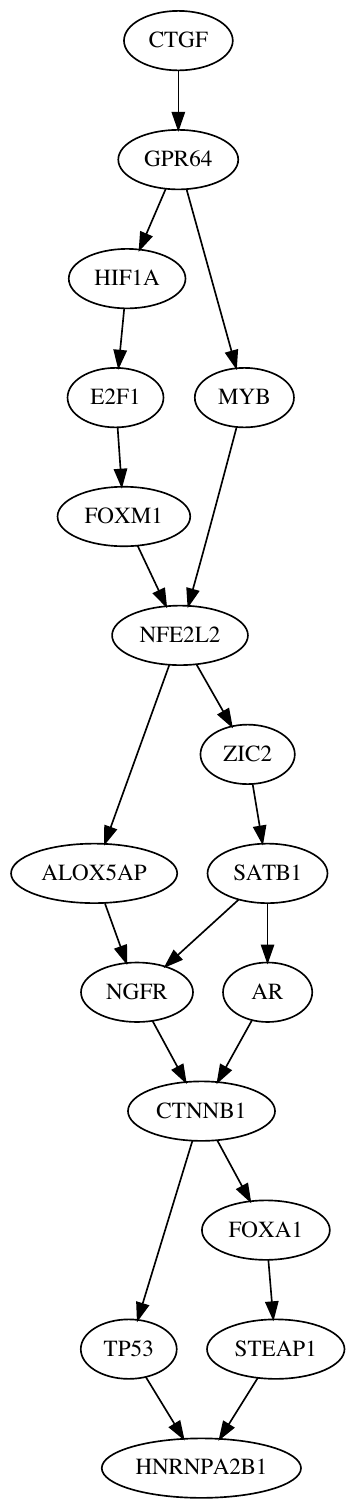}  
	\end{minipage}

	\caption{DAG structure recovered from interventional data in \cite{xiao2015gene}. (Left) Nodes correspond to mouse genes. (Right) Nodes correspond to human genes.}
	\label{fig:mouse}
\end{figure}

	\end{appendices}

\end{document}